\documentclass[letterpaper,11pt]{article}

\usepackage{fullpage}
\usepackage[margin=1in]{geometry}
\usepackage{caption}
\usepackage[utf8]{inputenc} 
\usepackage[T1]{fontenc}    
\usepackage{hyperref}       
\usepackage{url}            
\usepackage{booktabs}       
\usepackage{amsfonts}       
\usepackage{nicefrac}       
\usepackage{microtype}      
\usepackage{amssymb}
\usepackage{graphicx}
\usepackage{color}
\usepackage{natbib}
\usepackage{complexity}
\usepackage{subcaption}
\usepackage{amsmath}
\usepackage{amsthm}
\usepackage{tabularx}
\usepackage{fancyvrb}
\usepackage{comment}
\usepackage{algorithm}
\usepackage{algorithmic}
\usepackage{float}
\usepackage{wrapfig}
\usepackage{mathtools}
\usepackage{thmtools,thm-restate}
\usepackage{paralist} 
\usepackage{multirow}
\usepackage{scrextend}
\usepackage{enumitem}
\usepackage{quoting}
\usepackage{mathpazo}
\usepackage{hyperref}
\usepackage{thm-restate}
\usepackage{tabularx}
\usepackage{xcolor}
\usepackage{xspace}

\theoremstyle{definition}
\newtheorem{definition}{Definition}[section]

\newtheorem{theorem}{Theorem}[section]
\newtheorem{lemma}{Lemma}[section]

\newtheorem{corollary}{Corollary}[section]

\newcommand{\domain}{\mathcal{D}}
\newcommand{\empdomain}{\widehat{\mathcal{D}}}
\newcommand{\data}{\mathcal{X}}
\newcommand{\hdata}{\mathcal{Z}}
\newcommand{\ydata}{\mathcal{Y}}

\DeclareMathOperator{\sgn}{sgn}
\newcommand{\err}{\varepsilon}
\newcommand{\emperr}{\widehat{\varepsilon}}
\newcommand{\RR}{\mathbb{R}}

\newcommand{\Exp}{\mathbb{E}}
\newcommand{\HH}{\mathcal{H}}
\newcommand{\xx}{\mathbf{x}}
\newcommand{\sample}{\mathbf{S}}

\DeclareMathOperator*{\argmin}{arg\,min}

\newcommand{\defeq}{\vcentcolon=}

\newcommand{\eps}{\varepsilon}

\newcommand{\emrad}{\text{Rad}}
\newcommand{\sigmas}{\pmb\sigma}
\newcommand{\ind}{\mathbb{I}}
\DeclarePairedDelimiterX{\inp}[2]{\langle}{\rangle}{#1, #2}
\newcommand{\djs}{d_{\text{JS}}}
\newcommand{\jsd}{D_{\text{JS}}}
\newcommand{\kl}{D_{\text{KL}}}

\definecolor{dkgreen}{rgb}{0,0.6,0}
\definecolor{gray}{rgb}{0.5,0.5,0.5}
\definecolor{mauve}{rgb}{0.58,0,0.82}
\setlist{nolistsep}

\title{On Learning Invariant Representation for Domain Adaptation}

\author{Han Zhao$^\star$, Remi Tachet des Combes$^\dagger$, Kun Zhang$^\star$, Geoffrey J. Gordon$^{\star,\dagger}$\\Carnegie Mellon University$^\star$, Microsoft Research Montreal$^\dagger$ \\ han.zhao@cs.cmu.edu, kunz1@cmu.edu, \{remi.tachet, geoff.gordon\}@microsoft.com}
\date{}

\begin{document}

\maketitle


\begin{abstract}
    Due to the ability of deep neural nets to learn rich representations, recent advances in unsupervised domain adaptation have focused on learning domain-invariant features that achieve a small error on the source domain. The hope is that the learnt representation, together with the hypothesis learnt from the source domain, can generalize to the target domain. In this paper, we first construct a simple counterexample showing that, contrary to common belief, the above conditions are not sufficient to guarantee successful domain adaptation. In particular, the counterexample exhibits \emph{conditional shift}: the class-conditional distributions of input features change between source and target domains. To give a sufficient condition for domain adaptation, we propose a natural and interpretable generalization upper bound that explicitly takes into account the aforementioned shift. Moreover, we shed new light on the problem by proving an information-theoretic lower bound on the joint error of \emph{any} domain adaptation method that attempts to learn invariant representations. Our result characterizes a fundamental tradeoff between learning invariant representations and achieving small joint error on both domains when the marginal label distributions differ from source to target. Finally, we conduct experiments on real-world datasets that corroborate our theoretical findings. We believe these insights are helpful in guiding the future design of domain adaptation and representation learning algorithms.
\end{abstract}

\section{Introduction}
\label{sec:intro}
The recent successes of supervised deep learning methods have been partially attributed to rich datasets and increasing computational power. However, in many critical applications, e.g., self-driving cars or personal healthcare, it is often prohibitively expensive and time-consuming to collect large-scale supervised training data. Unsupervised domain adaptation (DA) focuses on such limitations by trying to transfer knowledge from a labeled source domain to an unlabeled target domain, and a large body of work tries to achieve this by exploring domain-invariant structures and representations to bridge the gap. Theoretical results~\citep{ben2010theory,mansour2009domain,mansour2012robust} and algorithms~\citep{glorot2011domain,becker2013non,ajakan2014domain,adel2017unsupervised,pei2018multi,zhao2017efficient} under this setting are abundant. 

\begin{figure*}[htb]
    \centering
    \includegraphics[width=\linewidth]{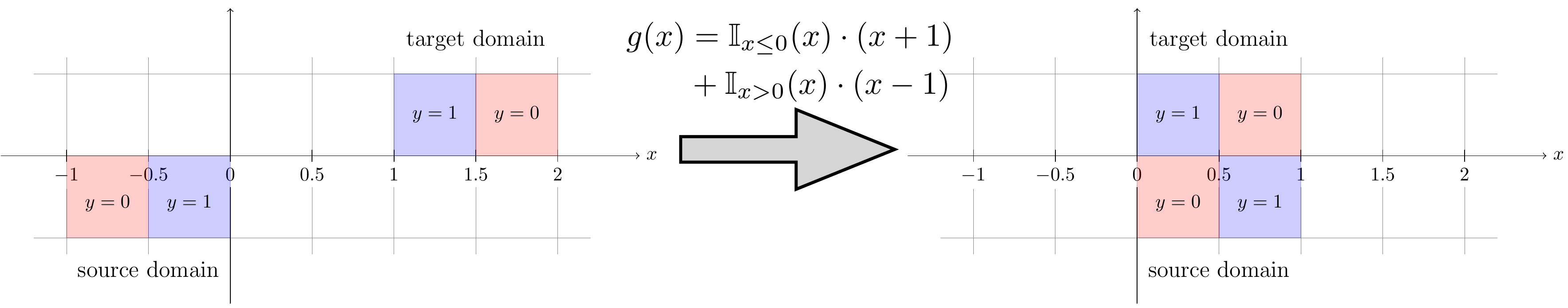}
    \caption{A counterexample where invariant representations lead to large joint error on source and target domains. Before transformation of $g(\cdot)$, $h^*(x) = 1$ iff $x\in (-1/2, 3/2)$ achieves perfect classification on both domains. After transformation, source and target distributions are perfectly aligned, but no hypothesis can achieve a small joint error.}
    \label{fig:example}
\end{figure*}
Due to the ability of deep neural nets to learn rich feature representations, recent advances in domain adaptation have focused on using these networks to learn \emph{invariant representations}, i.e., intermediate features whose distribution is the same in source and target domains, while at the same time achieving small error on the source domain. The hope is that the learnt intermediate representation, together with the hypothesis learnt using labeled data from the source domain, can generalize to the target domain. Nevertheless, from a theoretical standpoint, it is not at all clear whether aligned representations and small source error are sufficient to guarantee good generalization on the target domain. In fact, despite being successfully applied in various applications~\citep{zhang2017aspect,hoffman2017cycada}, it has also been reported that such methods fail to generalize in certain closely related source/target pairs, e.g., digit classification from MNIST to SVHN~\citep{ganin2016domain}.

Given the wide application of domain adaptation methods based on learning invariant representations, we attempt in this paper to answer the following important and intriguing question: 
\begin{quoting}
\itshape
    Is finding invariant representations while at the same time achieving a small source error sufficient to guarantee a small target error? If not, under what conditions is it?
\end{quoting}
Contrary to common belief, we give a negative answer to the above question by constructing a simple example showing that these two conditions are not sufficient to guarantee target generalization, even in the case of perfectly aligned representations between the source and target domains. In fact, our example shows that the objective of learning invariant representations while minimizing the source error can actually be hurtful, in the sense that the better the objective, the larger the target error. At a colloquial level, this happens because learning invariant representations can break the originally favorable underlying problem structure, i.e., close labeling functions and conditional distributions. To understand when such methods work, we propose a generalization upper bound as a sufficient condition that explicitly takes into account the conditional shift between source and target domains. The proposed upper bound admits a natural interpretation and decomposition in domain adaptation; we show that it is tighter than existing results in certain cases. 

Simultaneously, to understand what the necessary conditions for representation based approaches to work are, we prove an information-theoretic lower bound on the joint error of both domains for \emph{any} algorithm based on learning invariant representations. Our result complements the above upper bound and also extends the constructed example to more general settings. The lower bound sheds new light on this problem by characterizing a fundamental tradeoff between learning invariant representations and achieving small joint error on both domains when the marginal label distributions differ from source to target. Our lower bound directly implies that minimizing source error while achieving invariant representation will only increase the target error. We conduct experiments on real-world datasets that corroborate this theoretical implication. Together with the generalization upper bound, our results suggest that adaptation should be designed to align the label distribution as well when learning an invariant representation (c.f. Sec.~\ref{sec:lowerbound}). We believe these insights will be helpful to guide the future design of domain adaptation and representation learning algorithms.

\section{Preliminary}
\label{sec:preliminary}
We first introduce the notations used throughout this paper and review a theoretical model for domain adaptation (DA)~\citep{kifer2004detecting,ben2007analysis,blitzer2008learning,ben2010theory}. 

\textbf{Notations}~~We use $\data$ and $\ydata$ to denote the input and output space, respectively. Similarly, $\hdata$ stands for the representation space induced from $\data$ by a feature transformation $g:\data\mapsto\hdata$. Accordingly, we use $X, Y, Z$ to denote the random variables which take values in $\data, \ydata,\hdata$, respectively. In this work, \emph{domain} corresponds to a distribution $\domain$ on the input space $\data$ and a labeling function $f:\data\to [0, 1]$. In the domain adaptation setting, we use $\langle \domain_S, f_S\rangle$ and $\langle \domain_T, f_T\rangle$ to denote the source and target domains, respectively. A \emph{hypothesis} is a function $h:\data\to \{0, 1\}$. The \emph{error} of a hypothesis $h$ w.r.t.\ the labeling function $f$ under distribution $\domain_S$ is defined as: $\eps_S(h, f)\defeq \Exp_{\xx\sim\domain_S}[|h(\xx) - f(\xx)|]$. When $f$ and $h$ are binary classification functions, this definition reduces to the probability that $h$ disagrees with $f$ under $\domain_S$: $\Exp_{\xx\sim\domain_S}[|h(\xx) - f(\xx)|] = \Exp_{\xx\sim\domain_S}[\ind(f(\xx)\neq h(\xx))] = \Pr_{\xx\sim\domain_S}(f(\xx)\neq h(\xx))$. In this work, we focus on the deterministic setting where the output $Y = f(X)$ is given by a deterministic labeling function $f$ defined on the corresponding domain. For two functions $g$ and $h$ with compatible domains and ranges, we use $h\circ g$ to denote the function composition $h(g(\cdot))$. Other notations will be introduced in the context when necessary.

\subsection{Problem Setup}
We consider the unsupervised domain adaptation problem where the learning algorithm has access to a set of $n$ labeled points $\{(\xx_i, y_i)\}_{i=1}^n\in(\data\times\ydata)^n$ sampled i.i.d.\ from the source domain and a set of unlabeled points $\{\xx_j\}_{j=1}^m\in\data^m$ sampled i.i.d.\ from the target domain. At a colloquial level, the goal of an unsupervised domain adaptation algorithm is to generalize well on the target domain by learning from labeled samples from the source domain as well as unlabeled samples from the target domain. Formally, let the \emph{risk} of hypothesis $h$ be the error of $h$ w.r.t.\ the true labeling function under domain $\mathcal{D}_S$, i.e., $\eps_S(h)\defeq \eps_S(h, f_S)$. As commonly used in computational learning theory, we denote by $\emperr_S(h)$ the empirical risk of $h$ on the source domain. Similarly, we use $\eps_T(h)$ and $\emperr_T(h)$ to mean the true risk and the empirical risk on the target domain. The problem of domain adaptation considered in this work can be stated as: under what conditions and by what algorithms can we guarantee that a small training error $\emperr_S(h)$ implies a small test error $\err_T(h)$? Clearly, this goal is not always possible if the source and target domains are far away from each other. 

\subsection{A Theoretical Model for Domain Adaptation}
To measure the similarity between two domains, it is crucial to define a discrepancy measure between them. To this end,~\citet{ben2010theory} proposed the $\HH$-divergence to measure the distance between two distributions:
\begin{definition}[$\HH$-divergence]
Let $\HH$ be a hypothesis class on input space $\data$, and $\mathcal{A}_{\HH}$ be the collection of subsets of $\data$ that are the support of some hypothesis in $\HH$, i.e., $\mathcal{A}_{\HH}\defeq \{h^{-1}(1)\mid h\in\HH\}$. The distance between two distributions $\domain$ and $\domain'$ based on $\HH$ is: $d_{\HH}(\domain, \domain')\defeq \sup_{A\in\mathcal{A}_{\HH}}|\Pr_{\domain}(A) - \Pr_{\domain'}(A)|$.~\footnote{To be precise,~\citet{ben2007analysis}'s original definition of $\HH$-divergence has a factor of 2, we choose the current definition as the constant factor is inessential.}
\end{definition}

$\HH$-divergence is particularly favorable in the analysis of domain adaptation with binary classification problems, and it had also been generalized to the \emph{discrepancy distance}~\citep{cortes2008sample,mansour2009domain,mansour2009multiple,cortes2014domain} for general loss functions, including the one for regression problems. Both $\HH$-divergence and the discrepancy distance can be estimated using finite unlabeled samples from both domains when $\HH$ has a finite VC-dimension. 

One flexibility of the $\HH$-divergence is that its power on measuring the distance between two distributions can be controlled by the richness of the hypothesis class $\HH$. To see this, first consider the situation where $\HH$ is very restrictive so that it only contains the constant functions $h\equiv 0$ and $h\equiv 1$. In this case, it can be readily verified by the definition that $d_{\HH}(\domain, \domain') = 0,~\forall~\domain, \domain'$. On the other extreme, if $\HH$ contains all the measurable binary functions, then $d_{\HH}(\domain, \domain') = 0$ iff $\domain(\cdot) = \domain'(\cdot)$ almost surely. In this case the $\HH$-divergence reduces to the total variation, or equivalently the $L_1$ distance, between the two distributions. 

Given a hypothesis class $\HH$, we define its symmetric difference w.r.t.\ itself as: $\HH\Delta\HH = \{h(\xx)\oplus h'(\xx)\mid h, h'\in\HH\}$, where $\oplus$ is the \texttt{xor} operation. Let $h^*$ be the optimal hypothesis that achieves the minimum joint risk on both the source and target domains: $h^*\defeq \argmin_{h\in\HH}\err_S(h) + \err_T(h)$, and let $\lambda^*$ denote the joint risk of the optimal hypothesis $h^*$: $\lambda^* \defeq \err_S(h^*) + \err_T(h^*)$. \citet{ben2007analysis} proved the following generalization bound on the target risk in terms of the empirical source risk and the discrepancy between the source and target domains:
\begin{theorem}[\citet{ben2007analysis}]
\label{thm:shai}
Let $\HH$ be a hypothesis space of VC-dimension $d$ and $\empdomain_S$ (resp. $\empdomain_T$) be the empirical distribution induced by a sample of size $n$ drawn from $\domain_S$ (resp. $\domain_T$). Then w.p.\ at least $1 - \delta$, $\forall h\in\HH$, 
\begin{align}
\err_T(h) \leq &~\emperr_S(h) + \frac{1}{2}d_{\HH\Delta\HH}(\empdomain_S, \empdomain_T) + \lambda^* +  O\left(\sqrt{\frac{d\log n + \log(1/\delta)}{n}}\right).
\label{equ:singlebound}
\end{align}
\end{theorem}
The bound depends on $\lambda^*$, the optimal joint risk that can be achieved by the hypotheses in $\HH$. The intuition is the following: if $\lambda^*$ is large, we cannot hope for a successful domain adaptation. Later in Sec.~\ref{sec:lowerbound}, we shall get back to this term to show an information-theoretic lower bound on it for any approach based on learning invariant representations.

Theorem~\ref{thm:shai} is the foundation of many recent works on unsupervised domain adaptation via learning invariant representations~\citep{ajakan2014domain,ganin2016domain,zhao2018adversarial,pei2018multi,zhao2018multiple}. It has also inspired various applications of domain adaptation with adversarial learning, e.g., video analysis~\citep{hoffman2016fcns,shrivastava2016learning,hoffman2017cycada,tzeng2017adversarial}, natural language understanding~\citep{zhang2017aspect,fu2017domain}, speech recognition~\citep{zhao18deep,hosseini2018augmented}, to name a few.

At a high level, the key idea is to learn a rich and parametrized feature transformation $g:\data\mapsto\hdata$ such that the induced source and target distributions (on $\hdata$) are close, as measured by the $\HH$-divergence. We call $g$ an \emph{invariant representation} w.r.t.\ $\HH$ if $d_\HH(\domain_S^g, \domain_T^g) = 0$, where $\domain_S^g/\domain_T^g$ is the induced source/target distribution. At the same time, these algorithms also try to find new hypothesis (on the representation space $\hdata$) to achieve a small empirical error on the source domain. As a whole algorithm, these two procedures corresponds to simultaneously finding invariant representations and hypothesis to minimize the first two terms in the generalization upper bound of Theorem~\ref{thm:shai}. 

\section{Related Work}
\label{sec:related}
A number of adaptation approaches based on learning invariant representations have been proposed in recent years. Although in this paper we mainly focus on using the $\HH$-divergence to characterize the discrepancy between two distributions, other distance measures can be used as well, e.g., the maximum mean discrepancy (MMD)~\citep{long2014transfer,long2015learning,long2016unsupervised}, the Wasserstein distance~\citep{courty2017optimal,courty2017joint,shen2018wasserstein,lee2018minimax}, etc. 

Under the theoretical framework of the $\HH$-divergence, \citet{ganin2016domain} propose a domain adversarial neural network (DANN) to learn the domain invariant features. Adversarial training techniques that aim to build feature representations that are indistinguishable between source and target domains have been proposed in the last few years~\citep{ajakan2014domain,ganin2016domain}. Specifically, one of the central ideas is to use neural networks, which are powerful function approximators, to approximate the $\HH$-divergence between two domains~\citep{kifer2004detecting,ben2007analysis,ben2010theory}. The overall algorithm can be viewed as a zero-sum two-player game: one network tries to learn feature representations that can fool the other network, whose goal is to distinguish the representations generated on the source domain from those generated on the target domain. In a concurrent work,~\citet{johansson2019support} also identified the insufficiency of learning domain-invariant representation for successful adaptation. They further analyzed the information loss of non-invertible transformations, and proposed a generalization upper bound that directly takes it into account. In our work, by showing an information-theoretic lower bound on the joint error of these methods, we show that although invariant representations can be achieved, it does not necessarily translate to good generalization on the target domain, in particular when the label distributions of the two domains differ significantly. 

Causal approaches based on conditional and label shifts for domain adaptation also exist~\citep{zhang2013domain,gong2016domain,lipton2018detecting,azizzadenesheli2018regularized}. One typical assumption made to simplify the analysis in this line of work is that the source and target domains share the same generative distribution and only differ at the marginal label distributions. It is worth noting that~\citet{zhang2013domain} and \citet{gong2016domain} showed that both label and conditional shift can be successfully corrected when the changes in the generative distribution follow some parametric families. In this work we focus on representation learning and do not make such explicit assumptions.

\section{Theoretical Analysis}
\label{sec:theory}
Is finding invariant representations alone a sufficient condition for the success of domain adaptation? Clearly it is not. Consider the following simple counterexample: let $g_{\mathbf{c}}:\data\mapsto\hdata$ be a constant function, where $\forall\xx\in\data$, $g_{\mathbf{c}}(\xx) = \mathbf{c}\in\hdata$. Then for any discrepancy distance $d(\cdot, \cdot)$ over two distributions, including the $\HH$-divergence, MMD, and the Wasserstein distance, and for any distributions $\domain_S, \domain_T$ over the input space $\data$, we have $d(\domain_S^{g_{\mathbf{c}}}, \domain_T^{g_{\mathbf{c}}}) = 0$, where we use $\domain_S^{g_{\mathbf{c}}}$ (resp. $\domain_T^{g_{\mathbf{c}}}$) to mean the induced source (resp. target) distribution by the transformation $g_{\mathbf{c}}$ over the representation space $\hdata$. Furthermore, it is fairly easy to construct source and target domains $\langle\domain_S, f_S\rangle$, $\langle \domain_T, f_T\rangle$, such that for any hypothesis $h:\hdata\mapsto\ydata$, $\err_T(h\circ g_{\mathbf{c}})\geq 1/2$, while there exists a classification function $f:\data\to\ydata$ that achieves small error, e.g., the labeling function. 

One may argue, with good reason, that in the counterexample above, the empirical source error $\emperr_S(h\circ g_{\mathbf{c}})$ is also large with high probability. Intuitively, this is because the simple constant transformation function $g_{\mathbf{c}}$ fails to retain the discriminative information about the classification task at hand, despite the fact that it can construct invariant representations.

Is finding invariant representations and achieving a small source error sufficient to guarantee small target error? In this section we first give a negative answer to this question by constructing a counterexample where there exists a nontrivial transformation function $g:\data\mapsto\hdata$ and hypothesis $h:\hdata\mapsto\ydata$ such that both $\err_S(h\circ g)$ and $d_{\HH\Delta\HH}(\domain_S^g, \domain_T^g)$ are small, while at the same time the target error $\err_T(h\circ g)$ is large. Motivated by this negative result, we proceed to prove a generalization upper bound that explicitly characterizes a sufficient condition for the success of domain adaptation. We then complement the upper bound by showing an information-theoretic lower bound on the joint error of \emph{any} domain adaptation approach based on learning invariant representations.

\subsection{Invariant Representation and Small Source Risk are Not Sufficient}
\label{sec:example}
In this section, we shall construct a simple 1-dimensional example where there exists a function $h^*:\RR\mapsto \{0, 1\}$ that achieves zero error on \emph{both} source and target domains. Simultaneously, we show that there exists a transformation function $g:\RR\mapsto\RR$ under which the induced source and target distributions are perfectly aligned, but \emph{every} hypothesis $h:\RR\mapsto\{0, 1\}$ incurs a large joint error on the induced source and target domains. The latter further implies that if we find a hypothesis that achieves small error on the source domain, then it has to incur a large error on the target domain. We illustrate this example in Fig.~\ref{fig:example}.

Let $\data = \hdata = \RR$ and $\ydata = \{0, 1\}$. For $a \leq b$, we use $U(a, b)$ to denote the uniform distribution over $[a, b]$. Consider the following source and target domains:
\begin{align*}
    & \domain_S = U(-1, 0), && f_S(x) = \begin{cases}
    0, & x \leq -1/2 \\
    1, & x > -1/2
    \end{cases} \\
    & \domain_T = U(1, 2), && f_T(x) = \begin{cases}
    0, & x \geq 3/2 \\
    1, & x < 3/2
    \end{cases}
\end{align*}
In the above example, it is easy to verify that the interval hypothesis $h^*(x) = 1$ iff $x\in (-1/2, 3/2)$ achieves perfect classification on \emph{both} domains. Now consider the following transformation:
\begin{equation*}
    g(x) = \ind_{x\leq 0}(x)\cdot(x + 1) + \ind_{x > 0}(x)\cdot(x - 1).
\end{equation*}
Since $g(\cdot)$ is a piecewise linear function, it follows that $\domain_S^Z = \domain_T^Z = U(0, 1)$, and for any distance metric $d(\cdot, \cdot)$ over distributions, we have $d(\domain_S^Z, \domain_T^Z) = 0$. But now for any hypothesis $h:\RR\mapsto \{0,1\}$, and $\forall x\in[0, 1]$, $h(x)$ will make an error in exactly one of the domains, hence
\begin{equation*}
    \forall h:\RR\mapsto\{0, 1\},\quad \err_S(h\circ g) + \err_T(h\circ g) = 1.
\end{equation*}
In other words, under the above invariant transformation $g$, the smaller the source error, the larger the target error. 

One may argue that this example seems to contradict the generalization upper bound from Theorem~\ref{thm:shai}, where the first two terms correspond exactly to a small source error and an invariant representation. The key to explain this apparent contradiction lies in the third term of the upper bound, $\lambda^*$, i.e., the optimal joint error achievable on \emph{both} domains. In our example, when there is no transformation applied to the input space, we show that $h^*$ achieves 0 error on both domains, hence $\lambda^* = \min_{h\in\HH} \err_S(h) + \err_T(h) = 0$. However, when the transformation $g$ is applied to the original input space, we prove that every hypothesis has joint error 1 on the representation space, hence $\lambda^*_g = 1$. Since we usually do not have access to the optimal hypothesis on both domains, although the generalization bound still holds on the representation space, it becomes vacuous in our example.

An alternative way to interpret the failure of the constructed example is that the labeling functions (or conditional distributions in the stochastic setting) of source and target domains are far away from each other in the representation space. Specifically, in the induced representation space, the optimal labeling function on the source and target domains are:
\begin{equation*}
    f'_S(x) = \begin{cases}
    0, & x \leq 1/2\\
    1, & x > 1/2
    \end{cases},\quad 
    f'_T(x) = \begin{cases}
    0, & x > 1/2 \\
    1, & x \leq 1/2
    \end{cases}, 
\end{equation*}
and we have $||f'_S - f'_T||_1 = \Exp_{x\sim U(0,1)}[|f'_S(x) - f'_T(x)|] = 1$.

\subsection{A Generalization Upper Bound}
\label{sec:upperbound}
For most of the practical hypothesis spaces $\HH$, e.g., half spaces, it is usually intractable to compute the optimal joint error $\lambda^*$ from Theorem~\ref{thm:shai}. Furthermore, the fact that $\lambda^*$ contains errors from both domains makes the bound very conservative and loose in many cases. In this section, inspired by the constructed example from Sec.~\ref{sec:example}, we aim to provide a general, intuitive, and interpretable generalization upper bound for domain adaptation that is free of the pessimistic $\lambda^*$ term. Ideally, the bound should also explicitly characterize how the shift between labeling functions of both domains affects domain adaptation. Due to space constraints, we refer the interested reader to the Appendix for the proofs of our technical lemmas, and mainly focus in the following on interpretations and results.

Because of its flexibility in choosing the witness function class $\HH$ and its natural interpretation as adversarial binary classification, we still adopt the $\HH$-divergence to measure the discrepancy between two distributions. For any hypothesis space $\HH$, it can be readily verified that $d_\HH(\cdot, \cdot)$ satisfies the triangular inequality:
\begin{equation*}
    d_\HH(\domain, \domain')\leq d_\HH(\domain, \domain'') + d_\HH(\domain'', \domain'), 
\end{equation*}
where $\domain, \domain', \domain''$ are any distributions over the same space. We now introduce a technical lemma that will be helpful in proving results related to the $\HH$-divergence:
\begin{restatable}{lemma}{technical}
Let $\HH\subseteq [0,1]^{\data}$ and $\domain, \domain'$ be two distributions over $\data$. Then $\forall h, h'\in\HH$, $|\err_{\domain}(h, h') - \err_{\domain'}(h, h')|\leq d_{\tilde{\HH}}(\domain, \domain')$, where $\tilde{\HH}\defeq \{\sgn(|h(\xx) - h'(\xx)| - t)\mid h,h'\in\HH, 0\leq t \leq 1\}$.
\label{lemma:key}
\end{restatable}
As a matter of fact, the above lemma also holds for any function class $\HH$ (not necessarily a hypothesis space) where there exists a constant $M > 0$, such that $||h||_{\infty} \leq M$ for all $h\in\HH$. Another useful lemma is the following triangular inequality:
\begin{restatable}{lemma}{tri}
Let $\HH\subseteq [0,1]^{\data}$ and $\domain$ be any distribution over $\data$. For any $h, h', h''\in\HH$, we have $\err_{\domain}(h, h')\leq \err_{\domain}(h, h'') + \err_{\domain}(h'', h')$.
\label{lemma:tri}
\end{restatable}

Let $f_S: \data\to [0,1]$ and $f_T:\data\to [0,1]$ be the optimal labeling functions on the source and target domains, respectively. In the stochastic setting, $f_S(\xx) = \Pr_{S}(y=1\mid \xx)$ corresponds to the optimal Bayes classifier. With these notations, the following theorem holds:
\begin{restatable}{theorem}{population}
Let $\langle \domain_S, f_S\rangle$ and $\langle \domain_T, f_T\rangle$ be the source and target domains, respectively. For any function class $\HH\subseteq [0,1]^{\data}$, and $\forall h\in\HH$, the following inequality holds: 
\begin{align*}
\err_T(h) \leq \err_S(h) + d_{\tilde{\HH}}(\domain_S, \domain_T) + \min\{\Exp_{\domain_S}[|f_S - f_T|], \Exp_{\domain_T}[|f_S - f_T|]\}.
\end{align*}
\label{thm:populationbound}
\end{restatable}
\textbf{Remark}~~The three terms in the upper bound have natural interpretations: the first term is the source error, the second one corresponds to the discrepancy between the marginal distributions, and the third measures the distance between the labeling functions from the source and target domains. Altogether, they form a sufficient condition for the success of domain adaptation: besides a small source error, not only do the marginal distributions need to be close, but so do the labeling functions. 

\textbf{Comparison with Theorem~\ref{thm:shai}}.~It is instructive to compare the bound in Theorem~\ref{thm:populationbound} with the one in Theorem~\ref{thm:shai}. The main difference lies in the $\lambda^*$ in Theorem~\ref{thm:shai} and the $\min\{\Exp_{\domain_S}[|f_S - f_T|], \Exp_{\domain_T}[|f_S - f_T|]\}$ in Theorem~\ref{thm:populationbound}. $\lambda^*$ depends on the choice of the hypothesis class $\HH$, while our term does not. In fact, our quantity reflects the underlying structure of the problem, i.e., the conditional shift. Finally, consider the example given in the left panel of Fig.~\ref{fig:example}. It is easy to verify that we have $\min\{\Exp_{\domain_S}[|f_S - f_T|], \Exp_{\domain_T}[|f_S - f_T|]\} = 1/2$ in this case, while for a natural class of hypotheses, i.e., $\HH\defeq \{h(x) = 0 \Leftrightarrow a \leq x \leq b~|~a < b\}$, we have $\lambda^* = 1$. In that case, our bound is tighter than the one in Theorem~\ref{thm:shai}.

In the covariate shift setting, where we assume the conditional distributions of $Y\mid X$ between the source and target domains are the same, the third term in the upper bound vanishes. In that case the above theorem says that to guarantee successful domain adaptation, it suffices to match the marginal distributions while achieving small error on the source domain. In general settings where the optimal labeling functions of the source and target domains differ, the above bound says that it is not sufficient to simply match the marginal distributions and achieve small error on the source domain. At the same time, we should also guarantee that the optimal labeling functions (or the conditional distributions of both domains) are not too far away from each other. As a side note, it is easy to see that $\Exp_{\domain_S}[|f_S - f_T|] = \err_S(f_T)$ and $\Exp_{\domain_T}[|f_S - f_T|] = \err_T(f_S)$. In other words, they are essentially the cross-domain errors. When the cross-domain error is small, it implies that the optimal source (resp. target) labeling function generalizes well on the target (resp. source) domain. 

Both the error term $\err_S(h)$ and the divergence $d_{\tilde{H}}(\domain_S, \domain_T)$ in Theorem~\ref{thm:populationbound} are with respect to the true underlying distributions $\domain_S$ and $\domain_T$, which are not available to us during training. In the following, we shall use the Rademacher complexity to provide for both terms a data-dependent bound from empirical samples from $\domain_S$ and $\domain_T$.
\begin{definition}[Empirical Rademacher Complexity]
Let $\HH$ be a family of functions mapping from $\data$ to $[a,b]$ and $\sample = \{\xx_i\}_{i=1}^n$ a fixed sample of size $n$ with elements in $\data$. Then, the \emph{empirical Rademacher complexity} of $\HH$ with respect to the sample $X$ is defined as
\begin{equation*}
    \emrad_{\sample}(\HH) \defeq \Exp_{\sigmas}\bigg[\sup_{h\in\HH}\frac{1}{n}\sum_{i=1}^n \sigma_i h(\xx_i)\bigg],
\end{equation*}
where $\sigmas = \{\sigma_i\}_{i=1}^n$ and $\sigma_i$ are i.i.d.\ uniform random variables taking values in $\{+1, -1\}$. 
\end{definition}
With the empirical Rademacher complexity, we can show that w.h.p., the empirical source error $\emperr_S(h)$ cannot be too far away from the population error $\err_S(h)$ for all $h\in \HH$:
\begin{restatable}{lemma}{source}
Let $\HH\subseteq [0,1]^\data$, then for all $\delta > 0$, w.p.\ at least $1-\delta$, the following inequality holds for all $h\in\HH$: $\err_S(h) \leq \emperr_S(h) + 2\emrad_\sample(\HH) + 3\sqrt{\log(2/\delta)/2n}$.
\label{lemma:source}
\end{restatable}
Similarly, for any distribution $\domain$ over $\data$, let $\empdomain$ be its empirical distribution from sample $\sample\sim\domain^n$ of size $n$. Then for any two distributions $\domain$ and $\domain'$, we can also use the empirical Rademacher complexity to provide a data-dependent bound for the perturbation between $d_\HH(\domain, \domain')$ and $d_\HH(\empdomain, \empdomain')$:
\begin{restatable}{lemma}{ddd}
Let $\tilde{\HH}$, $\domain$ and $\empdomain$ be defined above, then for all $\delta > 0$, w.p.\ at least $1-\delta$, the following inequality holds for all $h\in\tilde{\HH}$: $\Exp_\domain[\ind_h] \leq \Exp_{\widehat{\domain}}[\ind_h] + 2\emrad_\sample(\tilde{\HH}) + 3\sqrt{\log(2/\delta)/2n}$.
\label{lemma:ind}
\end{restatable}
Since $\tilde{\HH}$ is a hypothesis class, by definition we have:
\begin{align*}
    d_{\tilde{\HH}}(\domain, \empdomain) = \sup_{A\in\mathcal{A}_{\tilde{\HH}}}|\Pr_{\domain}(A) - \Pr_{\widehat{\domain}}(A)| = \sup_{h\in\tilde{\HH}}|\Exp_{\domain}[\ind_h] - \Exp_{\widehat{\domain}}[\ind_h]|.
\end{align*}
Hence combining the above identity with Lemma~\ref{lemma:ind}, we immediately have w.p.\ at least $1-\delta$:
\begin{equation}
    d_{\tilde{\HH}}(\domain, \empdomain) \leq 2\emrad_\sample(\tilde{\HH}) + 3\sqrt{\log(2/\delta)/2n}.
\end{equation}
Now use a union bound and the fact that $d_{\tilde{\HH}}(\cdot, \cdot)$ satisfies the triangle inequality, we have:
\begin{restatable}{lemma}{hdiv}
Let $\tilde{\HH}$, $\domain, \domain'$ and $\empdomain, \empdomain'$ be defined above, then for $\forall \delta > 0$, w.p.\ at least $1-\delta$, for $\forall h\in\tilde{\HH}$: 
\begin{equation*}
    d_{\tilde{\HH}}(\domain, \domain') \leq d_{\tilde{\HH}}(\empdomain, \empdomain') + 4\emrad_\sample(\tilde{\HH}) + 6\sqrt{\log(4/\delta)/2n}.
\end{equation*}
\label{lemma:hdiv}
\end{restatable}
Combine Lemma~\ref{lemma:source}, Lemma~\ref{lemma:hdiv} and Theorem~\ref{thm:populationbound} with a union bound argument, we get the following main theorem that characterizes an upper bound for domain adaptation:
\begin{restatable}{theorem}{main}
Let $\langle \domain_S, f_S\rangle$ and $\langle \domain_T, f_T\rangle$ be the source and target domains, and let $\empdomain_S, \empdomain_T$ be the empirical source and target distributions constructed from sample $\sample = \{\sample_S, \sample_T\}$, each of size $n$. Then for any $\HH\subseteq [0,1]^{\data}$ and $\forall h\in\HH$:
\begin{align*}
    \err_T(h) \leq &~ \emperr_S(h) + d_{\tilde{\HH}}(\empdomain_S, \empdomain_T) + 2\emrad_\sample(\HH) + 4\emrad_\sample(\tilde{\HH}) + \min\{\Exp_{\domain_S}[|f_S - f_T|], \Exp_{\domain_T}[|f_S - f_T|]\} \\
    & + O\left(\sqrt{\log(1/\delta)/n}\right),
\end{align*}
where $\tilde{\HH}\defeq\{\sgn(|h(\xx) - h'(\xx)| - t)| h,h'\in\HH, t \in[0, 1]\}$.
\label{thm:main}
\end{restatable}
Essentially, the generalization upper bound can be decomposed into three parts: the first part comes from the domain adaptation setting, including the empirical source error, the empirical $\HH$-divergence, and the shift between labeling functions. The second part corresponds to complexity measures of our hypothesis space $\HH$ and $\tilde{\HH}$, and the last part describes the error caused by finite samples.

\subsection{An Information-Theoretic Lower Bound}
\label{sec:lowerbound}
In Sec.~\ref{sec:example}, we constructed an example to demonstrate that learning invariant representations could lead to a feature space where the joint error on both domains is large. In this section, we extend the example by showing that a similar result holds in more general settings. Specifically, we shall prove that for \emph{any} approach based on learning invariant representations, there is an intrinsic lower bound on the joint error of source and target domains, due to the discrepancy between their marginal label distributions. Our result hence highlights the need to take into account task related information when designing domain adaptation algorithms based on learning invariant representations.

Before we proceed to the lower bound, we first define several information-theoretic concepts that will be used in the analysis. For two distributions $\domain$ and $\domain'$, the Jensen-Shannon (JS) divergence $\jsd(\domain~||~\domain')$ is defined as:
\begin{equation*}
    \jsd(\domain~||~\domain') \defeq \frac{1}{2}\kl(\domain~||~\domain_M) + \frac{1}{2}\kl(\domain'~||~\domain_M),
\end{equation*}
where $\kl(\cdot~||~\cdot)$ is the Kullback–Leibler (KL) divergence and $\domain_M\defeq (\domain + \domain') / 2$. The JS divergence can be viewed as a symmetrized and smoothed version of the KL divergence, and it is closely related to the $L_1$ distance between two distributions through Lin's lemma~\citep{lin1991divergence}.

Unlike the KL divergence, the JS divergence is bounded: $0 \leq \jsd(\domain~||~\domain') \leq 1$. Additionally, from the JS divergence, we can define a distance metric between two distributions as well, known as the JS distance~\citep{endres2003new}:
\begin{equation*}
    \djs(\domain, \domain')\defeq\sqrt{\jsd(\domain~||~\domain')}.
\end{equation*}
With respect to the JS distance and for any (stochastic) mapping $h:\hdata\mapsto\ydata$, we can prove the following lemma via the celebrated data processing inequality:
\begin{restatable}{lemma}{dpi}
Let $\domain^Z_S$ and $\domain^Z_T$ be two distributions over $\hdata$ and let $\domain^Y_S$ and $\domain^Y_T$ be the induced distributions over $\ydata$ by function $h:\hdata\mapsto\ydata$, then
\begin{equation}
    \djs(\domain^Y_S, \domain^Y_T) \leq \djs(\domain^Z_S, \domain^Z_T).
\end{equation}
\label{lemma:jsd}
\end{restatable}
For methods that aim to learn invariant representations for domain adaptation, an intermediate representation space $\hdata$ is found through feature transformation $g$, based on which a common hypothesis $h:\hdata\mapsto\ydata$ is shared between both domains~\citep{ganin2016domain,tzeng2017adversarial,zhao2018adversarial}. Through this process, the following Markov chain holds:
\begin{equation}
    X \overset{g}{\longrightarrow} Z \overset{h}{\longrightarrow} \hat{Y},
\end{equation}
where $\hat{Y} = h(g(X))$ is the predicted random variable of interest. Hence for any distribution $\domain$ over $\data$, this Markov chain also induces a distribution $\domain^Z$ over $\hdata$ and $\domain^{\hat{Y}}$ over $\ydata$. By Lemma~\ref{lemma:jsd}, we know that $\djs(\domain^{\hat{Y}}_S, \domain^{\hat{Y}}_T) \leq \djs(\domain^Z_S, \domain^Z_T)$. With these notations, noting that the JS distance is a metric, the following inequality holds:
\begin{equation*}
    \djs(\domain_S^Y, \domain_T^Y) \leq \djs(\domain_S^Y, \domain_S^{\hat{Y}}) + \djs(\domain_S^{\hat{Y}}, \domain_T^{\hat{Y}}) + \djs(\domain_T^{\hat{Y}}, \domain_T^Y).
\end{equation*}
Combining the above inequality with Lemma~\ref{lemma:jsd}, we immediately have:
\begin{align}
    \djs(\domain_S^{Y}, \domain_T^{Y}) \leq &~\djs(\domain_S^Z, \domain_T^Z) + \djs(\domain_S^Y, \domain_S^{\hat{Y}}) + \djs(\domain_T^Y, \domain_T^{\hat{Y}}).
\label{equ:chain}
\end{align}
Intuitively, $\djs(\domain_S^Y, \domain_S^{\hat{Y}})$ and $\djs(\domain_T^Y, \domain_T^{\hat{Y}})$ measure the distance between the predicted label distribution and the ground truth label distribution on the source and target domain, respectively. With the help of Lemma~\ref{lemma:lin}, the following result establishes a relationship between $\djs(\domain^Y, \domain^{\hat{Y}})$ and the accuracy of the prediction function $h$:
\begin{restatable}{lemma}{relationship}
Let $Y = f(X)\in\{0, 1\}$ where $f(\cdot)$ is the labeling function and $\hat{Y} = h(g(X))\in\{0, 1\}$ be the prediction function, then $\djs(\domain^Y, \domain^{\hat{Y}})\leq \sqrt{\eps(h\circ g)}$.
\label{lemma:relationship}
\end{restatable}
We are now ready to present the key lemma of the section:
\begin{restatable}{lemma}{lowerbound}
Suppose the Markov chain $X \overset{g}{\longrightarrow} Z \overset{h}{\longrightarrow} \hat{Y}$ holds, then
\begin{align*}
\djs(\domain_S^{Y}, \domain_T^{Y}) \leq\djs(\domain_S^Z, \domain_T^Z) + \sqrt{\eps_S(h\circ g)} + \sqrt{\eps_T(h\circ g)}.
\end{align*}
\label{thm:lower}
\end{restatable}
\textbf{Remark}~~This lemma shows that if the marginal label distributions are significantly different between the source and target domains, then in order to achieve a small joint error, the induced distributions over $\hdata$ from source and target domains have to be significantly different as well. Put another way, if we are able to find an invariant representation such that $\djs(\domain_S^Z, \domain_T^Z) = 0$, then the joint error of the composition function $h\circ g$ has to be large: 
\begin{restatable}{theorem}{ll}
\label{thm:lowerbound}
Suppose the condition in Lemma~\ref{thm:lower} holds and $\djs(\domain_S^Y, \domain_T^Y)\geq \djs(\domain_S^Z, \domain_T^Z)$, then:
\begin{equation*}
    \err_S(h\circ g) + \err_T(h\circ g) \geq \frac{1}{2}\left(\djs(\domain_S^Y, \domain_T^Y) - \djs(\domain_S^Z, \domain_T^Z)\right)^2.
\end{equation*}
\end{restatable}
\textbf{Remark}~~The lower bound gives us a necessary condition on the success of any domain adaptation approach based on learning invariant representations: if the marginal label distributions are significantly different between source and target domains, then minimizing $\djs(\domain_S^Z, \domain_T^Z)$ and the source error $\err_S(h\circ g)$ will only increase the target error. In fact, Theorem~\ref{thm:lowerbound} can be extended to hold in the setting where different transformation functions are applied in source and target domains:
\begin{corollary}
    \label{coro:different}
    Let $g_S$, $g_T$ be the source and target transformation functions from $\data$ to $\hdata$. Suppose the condition in Lemma~\ref{thm:lower} holds and $\djs(\domain_S^Y, \domain_T^Y)\geq \djs(\domain_S^Z, \domain_T^Z)$, then:
    \begin{equation*}
        \err_S(h\circ g_S) + \err_T(h\circ g_T) \geq \frac{1}{2}\left(\djs(\domain_S^Y, \domain_T^Y) - \djs(\domain_S^Z, \domain_T^Z)\right)^2.
    \end{equation*}
\end{corollary}
Recent work has also explored using different transformation functions to achieve invariant representations~\citep{bousmalis2016domain,tzeng2017adversarial}, but Corollary~\ref{coro:different} shows that this is not going to help if the marginal label distributions differ between two domains. 

We conclude this section by noting that our bound on the joint error of both domains is not necessarily the tightest one. This can be seen from the example in Sec.~\ref{sec:example}, where $\djs(\domain_S^Z, \domain_T^Z) = \djs(\domain_S^Y, \domain_T^Y) = 0$, and we have $\err_S(h\circ g) + \err_T(h\circ g) = 1$, but in this case our result gives a trivial lower bound of 0. Nevertheless, our result still sheds new light on the importance of matching marginal label distributions in learning invariant representation for domain adaptation, which we believe to be a promising direction for the design of better adaptation algorithms.

\section{Experiments}
\label{sec:exp}
Our theoretical results on the lower bound of the joint error imply that over-training the feature transformation function and the discriminator may hurt generalization on the target domain. In this section, we conduct experiments on real-world datasets to verify our theoretical findings. The task is digit classification on three datasets of 10 classes: MNIST, USPS and SVHN. MNIST contains 60,000/10,000 train/test instances; USPS contains 7,291/2,007 train/test instances, and SVHN contains 73,257/26,032 train/test instances. We show the label distribution of these three datasets in Fig.~\ref{fig:label}.
\begin{figure}[htb]
    \centering
    \includegraphics[width=0.8\linewidth]{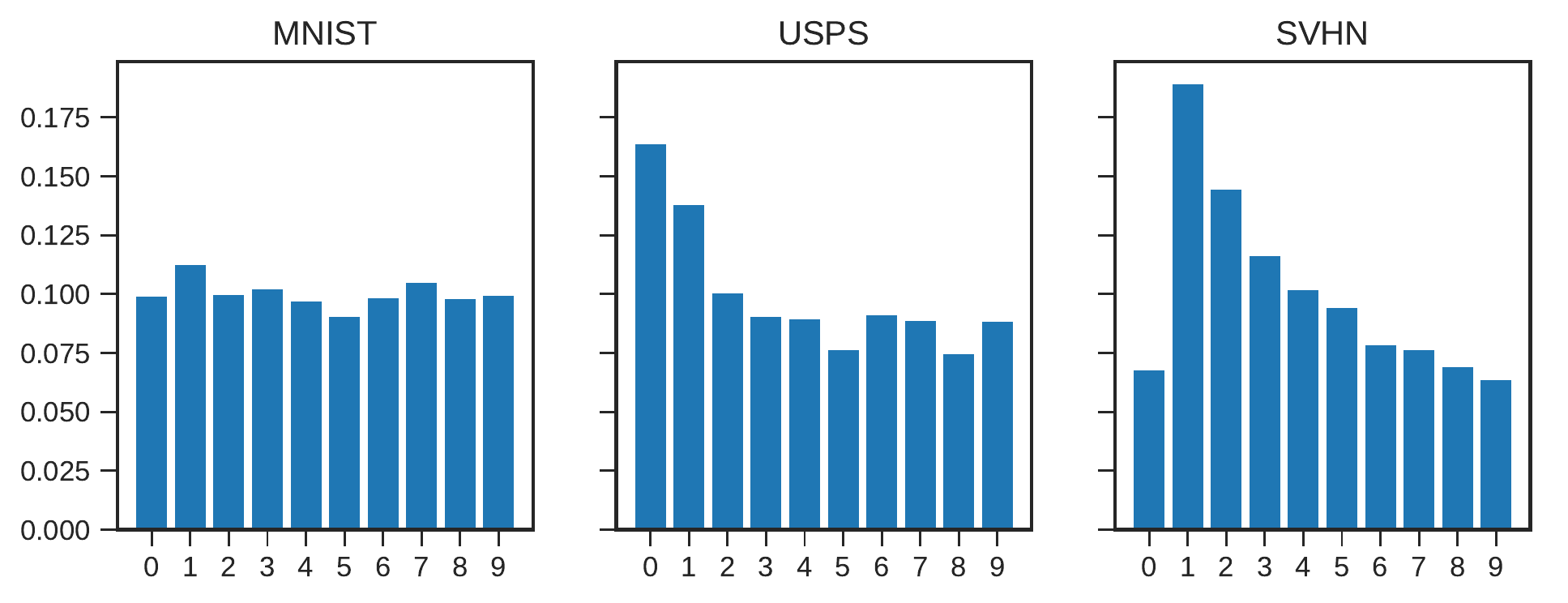}
    \caption{The label distributions of MNIST, USPS and SVHN.}
    \label{fig:label}
\end{figure}

Before training, we preprocess all the samples into gray scale single-channel images of size $16\times 16$, so they can be used by the same network. In our experiments, to ensure a fair comparison, we use the same network structure for all the experiments: 2 convolutional layers, one fully connected hidden layer, followed by a softmax output layer with 10 units. The convolution kernels in both layers are of size $5\times 5$, with 10 and 20 channels, respectively. The hidden layer has 1280 units connected to 100 units before classification. For domain adaptation, we use the original DANN~\citep{ganin2016domain} with gradient reversal implementation. The discriminator in DANN takes the output of convolutional layers as its feature input, followed by a $500\times 100$ fully connected layer, and a one-unit binary classification output. 
\begin{figure*}[htb]
    \centering
    \begin{subfigure}[b]{0.23\linewidth}
        \includegraphics[width=\linewidth]{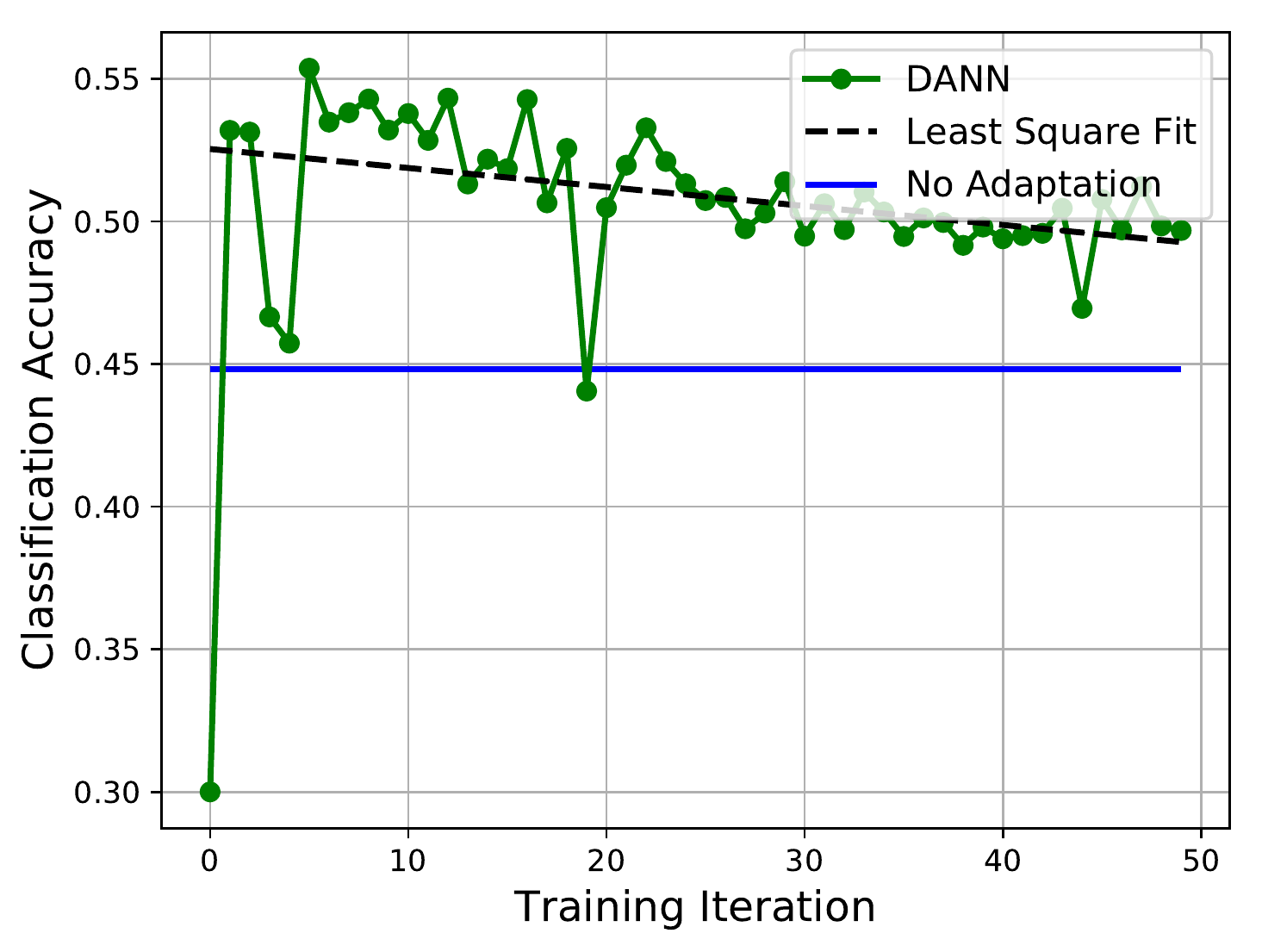}
        \caption{USPS $\to$ MNIST}
    \end{subfigure}
    ~
    \begin{subfigure}[b]{0.23\linewidth}
        \includegraphics[width=\linewidth]{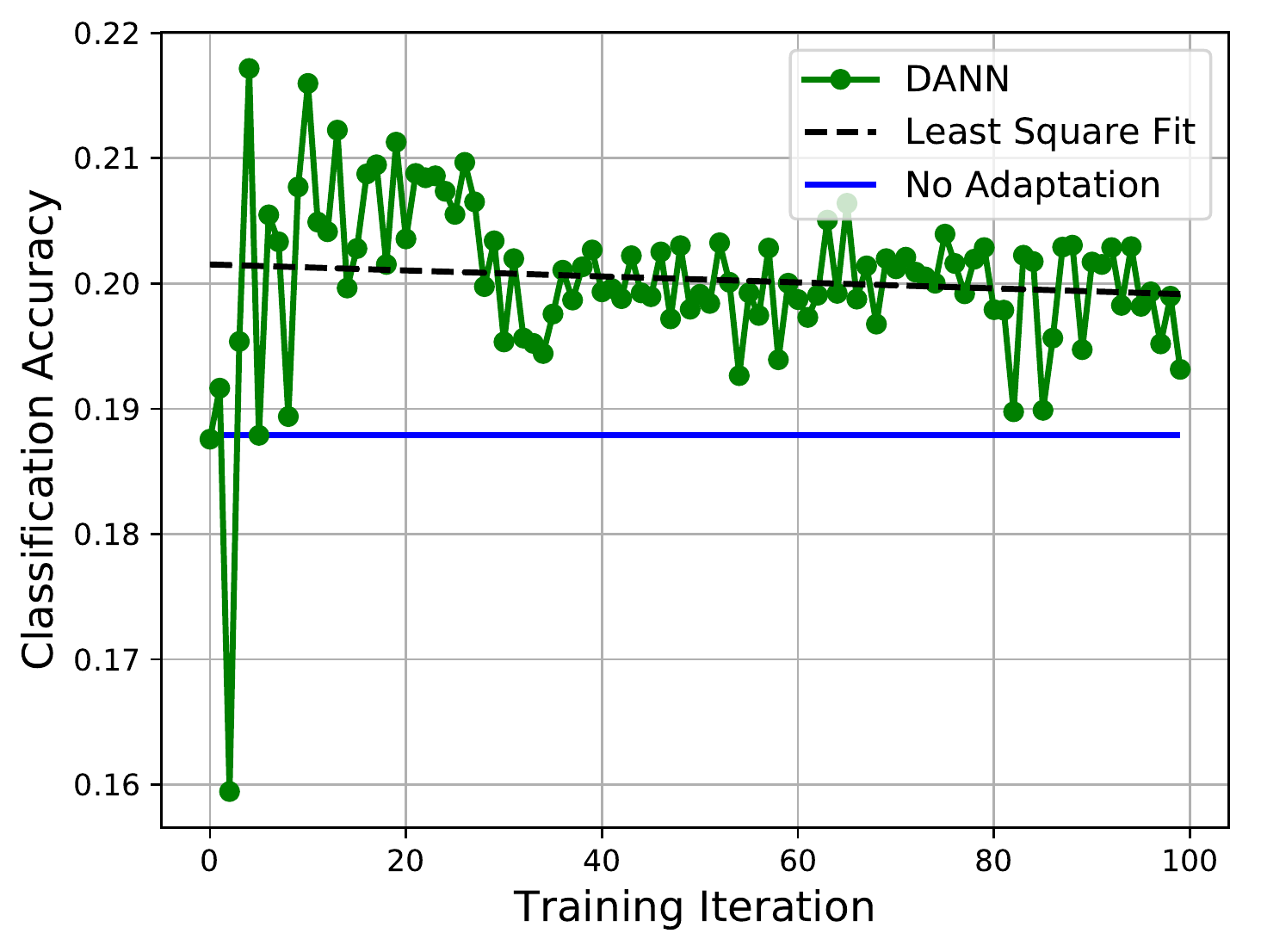}
        \caption{USPS $\to$ SVHN}
    \end{subfigure}
    ~
    \begin{subfigure}[b]{0.23\linewidth}
        \includegraphics[width=\linewidth]{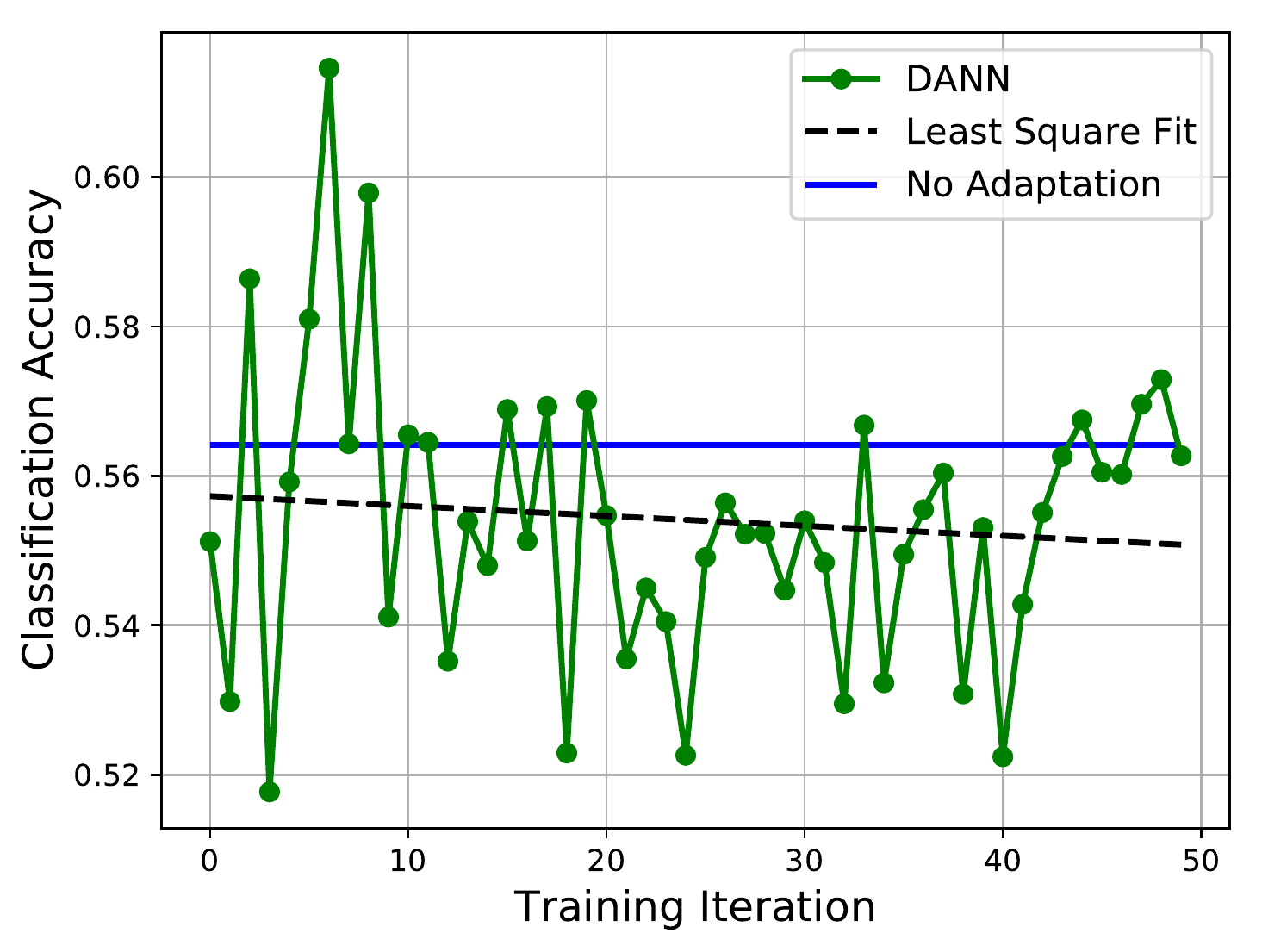}
        \caption{SVHN $\to$ MNIST}
    \end{subfigure}
    ~
    \begin{subfigure}[b]{0.23\linewidth}
        \includegraphics[width=\linewidth]{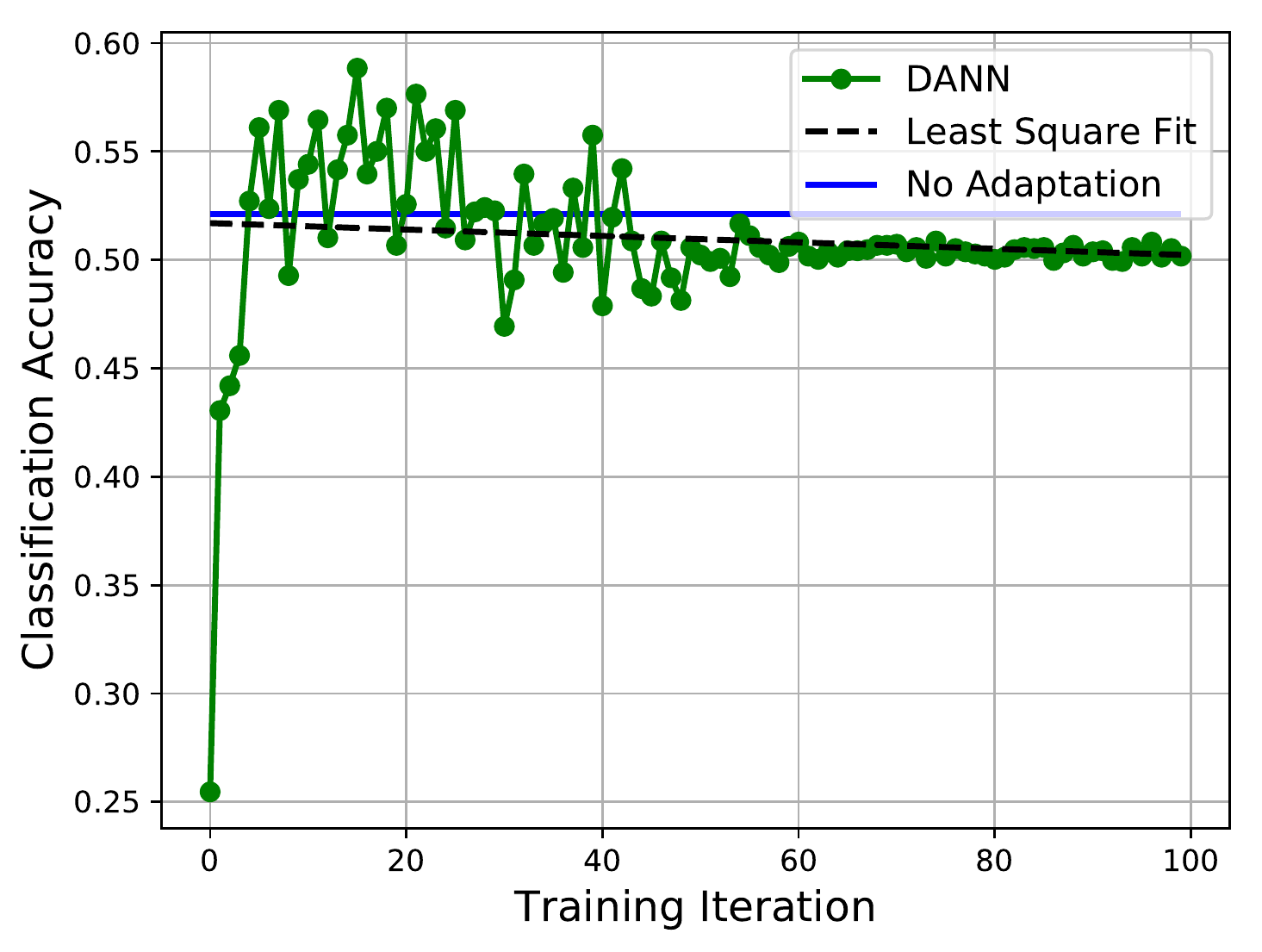}
        \caption{SVHN $\to$ USPS}
    \end{subfigure}
\caption{Digit classification on MNIST, USPS and SVHN. The horizontal solid line corresponds to the target domain test accuracy without adaptation. The green solid line is the target domain test accuracy under domain adaptation with DANN. We also plot the least square fit (dashed line) of the DANN adaptation results to emphasize the negative slope.}
\label{fig:dann}
\end{figure*}

We plot four adaptation trajectories in Fig.~\ref{fig:dann}. Among the four adaptation tasks, we can observe two phases in the adaptation accuracy. In the first phase, the test set accuracy rapidly grows, in less than 10 iterations. In the second phase, it gradually decreases after reaching its peak, despite the fact that the source training accuracy keeps increasing smoothly. Those phase transitions can be verified from the negative slopes of the least squares fit of the adaptation curves (dashed lines in Fig.~\ref{fig:dann}). We observe similar phenomenons on additional experiments using artificially unbalanced datasets trained on more powerful networks in Appendix~\ref{sec:additional}. The above experimental results imply that over-training the feature transformation and discriminator does not help generalization on the target domain, but can instead hurt it when the label distributions differ (as shown in Fig.~\ref{fig:label}). These experimental results are consistent with our theoretical findings. 

\section{Conclusion and Future Work}
\label{sec:conclusion}
In this paper we theoretically and empirically study the important problem of learning invariant representations for domain adaptation. We show that learning an invariant representation and achieving a small source error is not enough to guarantee target generalization. We then prove both upper and lower bounds for the target and joint errors, which directly translate to sufficient and necessary conditions for the success of adaptation. We believe our results take an important step towards understanding deep domain adaptation, and also stimulate future work on the design of stronger deep domain adaptation algorithms that align conditional distributions. Another interesting direction for future work is to characterize what properties the feature transformation function should have in order to decrease the conditional shift. It is also worth investigating under which conditions the label distributions can be aligned without explicit labeled data from the target domain.

\newpage
\bibliography{reference}

\begin{thebibliography}{42}
\providecommand{\natexlab}[1]{#1}
\providecommand{\url}[1]{\texttt{#1}}
\expandafter\ifx\csname urlstyle\endcsname\relax
  \providecommand{\doi}[1]{doi: #1}\else
  \providecommand{\doi}{doi: \begingroup \urlstyle{rm}\Url}\fi

\bibitem[Adel et~al.(2017)Adel, Zhao, and Wong]{adel2017unsupervised}
Tameem Adel, Han Zhao, and Alexander Wong.
\newblock Unsupervised domain adaptation with a relaxed covariate shift
  assumption.
\newblock In \emph{AAAI}, pages 1691--1697, 2017.

\bibitem[Ajakan et~al.(2014)Ajakan, Germain, Larochelle, Laviolette, and
  Marchand]{ajakan2014domain}
Hana Ajakan, Pascal Germain, Hugo Larochelle, Fran{\c{c}}ois Laviolette, and
  Mario Marchand.
\newblock Domain-adversarial neural networks.
\newblock \emph{arXiv preprint arXiv:1412.4446}, 2014.

\bibitem[Azizzadenesheli et~al.(2018)Azizzadenesheli, Liu, Yang, and
  Anandkumar]{azizzadenesheli2018regularized}
Kamyar Azizzadenesheli, Anqi Liu, Fanny Yang, and Animashree Anandkumar.
\newblock Regularized learning for domain adaptation under label shifts.
\newblock 2018.

\bibitem[Bartlett and Mendelson(2002)]{bartlett2002rademacher}
Peter~L Bartlett and Shahar Mendelson.
\newblock Rademacher and gaussian complexities: Risk bounds and structural
  results.
\newblock \emph{Journal of Machine Learning Research}, 3\penalty0
  (Nov):\penalty0 463--482, 2002.

\bibitem[Becker et~al.(2013)Becker, Christoudias, and Fua]{becker2013non}
Carlos~J Becker, Christos~M Christoudias, and Pascal Fua.
\newblock Non-linear domain adaptation with boosting.
\newblock In \emph{Advances in Neural Information Processing Systems}, pages
  485--493, 2013.

\bibitem[Ben-David et~al.(2007)Ben-David, Blitzer, Crammer, Pereira,
  et~al.]{ben2007analysis}
Shai Ben-David, John Blitzer, Koby Crammer, Fernando Pereira, et~al.
\newblock Analysis of representations for domain adaptation.
\newblock \emph{Advances in neural information processing systems},
  19:\penalty0 137, 2007.

\bibitem[Ben-David et~al.(2010)Ben-David, Blitzer, Crammer, Kulesza, Pereira,
  and Vaughan]{ben2010theory}
Shai Ben-David, John Blitzer, Koby Crammer, Alex Kulesza, Fernando Pereira, and
  Jennifer~Wortman Vaughan.
\newblock A theory of learning from different domains.
\newblock \emph{Machine learning}, 79\penalty0 (1-2):\penalty0 151--175, 2010.

\bibitem[Blitzer et~al.(2008)Blitzer, Crammer, Kulesza, Pereira, and
  Wortman]{blitzer2008learning}
John Blitzer, Koby Crammer, Alex Kulesza, Fernando Pereira, and Jennifer
  Wortman.
\newblock Learning bounds for domain adaptation.
\newblock In \emph{Advances in neural information processing systems}, pages
  129--136, 2008.

\bibitem[Bousmalis et~al.(2016)Bousmalis, Trigeorgis, Silberman, Krishnan, and
  Erhan]{bousmalis2016domain}
Konstantinos Bousmalis, George Trigeorgis, Nathan Silberman, Dilip Krishnan,
  and Dumitru Erhan.
\newblock Domain separation networks.
\newblock In \emph{Advances in Neural Information Processing Systems}, pages
  343--351, 2016.

\bibitem[Cortes and Mohri(2014)]{cortes2014domain}
Corinna Cortes and Mehryar Mohri.
\newblock Domain adaptation and sample bias correction theory and algorithm for
  regression.
\newblock \emph{Theoretical Computer Science}, 519:\penalty0 103--126, 2014.

\bibitem[Cortes et~al.(2008)Cortes, Mohri, Riley, and
  Rostamizadeh]{cortes2008sample}
Corinna Cortes, Mehryar Mohri, Michael Riley, and Afshin Rostamizadeh.
\newblock Sample selection bias correction theory.
\newblock In \emph{International Conference on Algorithmic Learning Theory},
  pages 38--53. Springer, 2008.

\bibitem[Courty et~al.(2017{\natexlab{a}})Courty, Flamary, Habrard, and
  Rakotomamonjy]{courty2017joint}
Nicolas Courty, R{\'e}mi Flamary, Amaury Habrard, and Alain Rakotomamonjy.
\newblock Joint distribution optimal transportation for domain adaptation.
\newblock In \emph{Advances in Neural Information Processing Systems}, pages
  3730--3739, 2017{\natexlab{a}}.

\bibitem[Courty et~al.(2017{\natexlab{b}})Courty, Flamary, Tuia, and
  Rakotomamonjy]{courty2017optimal}
Nicolas Courty, R{\'e}mi Flamary, Devis Tuia, and Alain Rakotomamonjy.
\newblock Optimal transport for domain adaptation.
\newblock \emph{IEEE transactions on pattern analysis and machine
  intelligence}, 39\penalty0 (9):\penalty0 1853--1865, 2017{\natexlab{b}}.

\bibitem[Endres and Schindelin(2003)]{endres2003new}
Dominik~Maria Endres and Johannes~E Schindelin.
\newblock A new metric for probability distributions.
\newblock \emph{IEEE Transactions on Information theory}, 2003.

\bibitem[Fu et~al.(2017)Fu, Nguyen, Min, and Grishman]{fu2017domain}
Lisheng Fu, Thien~Huu Nguyen, Bonan Min, and Ralph Grishman.
\newblock Domain adaptation for relation extraction with domain adversarial
  neural network.
\newblock In \emph{Proceedings of the Eighth International Joint Conference on
  Natural Language Processing (Volume 2: Short Papers)}, volume~2, pages
  425--429, 2017.

\bibitem[Ganin et~al.(2016)Ganin, Ustinova, Ajakan, Germain, Larochelle,
  Laviolette, Marchand, and Lempitsky]{ganin2016domain}
Yaroslav Ganin, Evgeniya Ustinova, Hana Ajakan, Pascal Germain, Hugo
  Larochelle, Fran{\c{c}}ois Laviolette, Mario Marchand, and Victor Lempitsky.
\newblock Domain-adversarial training of neural networks.
\newblock \emph{Journal of Machine Learning Research}, 17\penalty0
  (59):\penalty0 1--35, 2016.

\bibitem[Glorot et~al.(2011)Glorot, Bordes, and Bengio]{glorot2011domain}
Xavier Glorot, Antoine Bordes, and Yoshua Bengio.
\newblock Domain adaptation for large-scale sentiment classification: A deep
  learning approach.
\newblock In \emph{Proceedings of the 28th international conference on machine
  learning (ICML-11)}, pages 513--520, 2011.

\bibitem[Gong et~al.(2016)Gong, Zhang, Liu, Tao, Glymour, and
  Sch{\"o}lkopf]{gong2016domain}
Mingming Gong, Kun Zhang, Tongliang Liu, Dacheng Tao, Clark Glymour, and
  Bernhard Sch{\"o}lkopf.
\newblock Domain adaptation with conditional transferable components.
\newblock In \emph{International conference on machine learning}, pages
  2839--2848, 2016.

\bibitem[Hoffman et~al.(2016)Hoffman, Wang, Yu, and Darrell]{hoffman2016fcns}
Judy Hoffman, Dequan Wang, Fisher Yu, and Trevor Darrell.
\newblock Fcns in the wild: Pixel-level adversarial and constraint-based
  adaptation.
\newblock \emph{arXiv preprint arXiv:1612.02649}, 2016.

\bibitem[Hoffman et~al.(2017)Hoffman, Tzeng, Park, Zhu, Isola, Saenko, Efros,
  and Darrell]{hoffman2017cycada}
Judy Hoffman, Eric Tzeng, Taesung Park, Jun-Yan Zhu, Phillip Isola, Kate
  Saenko, Alexei~A Efros, and Trevor Darrell.
\newblock Cycada: Cycle-consistent adversarial domain adaptation.
\newblock \emph{arXiv preprint arXiv:1711.03213}, 2017.

\bibitem[Hosseini-Asl et~al.(2018)Hosseini-Asl, Zhou, Xiong, and
  Socher]{hosseini2018augmented}
Ehsan Hosseini-Asl, Yingbo Zhou, Caiming Xiong, and Richard Socher.
\newblock Augmented cyclic adversarial learning for domain adaptation.
\newblock \emph{arXiv preprint arXiv:1807.00374}, 2018.

\bibitem[Johansson et~al.(2019)Johansson, Ranganath, and
  Sontag]{johansson2019support}
Fredrik~D Johansson, Rajesh Ranganath, and David Sontag.
\newblock Support and invertibility in domain-invariant representations.
\newblock \emph{arXiv preprint arXiv:1903.03448}, 2019.

\bibitem[Kifer et~al.(2004)Kifer, Ben-David, and Gehrke]{kifer2004detecting}
Daniel Kifer, Shai Ben-David, and Johannes Gehrke.
\newblock Detecting change in data streams.
\newblock In \emph{Proceedings of the Thirtieth international conference on
  Very large data bases-Volume 30}, pages 180--191. VLDB Endowment, 2004.

\bibitem[Lee and Raginsky(2018)]{lee2018minimax}
Jaeho Lee and Maxim Raginsky.
\newblock Minimax statistical learning with wasserstein distances.
\newblock In \emph{Advances in Neural Information Processing Systems}, pages
  2692--2701, 2018.

\bibitem[Lin(1991)]{lin1991divergence}
Jianhua Lin.
\newblock {Divergence measures based on the Shannon entropy}.
\newblock \emph{IEEE Transactions on Information Theory}, 37\penalty0
  (1):\penalty0 145--151, 1991.

\bibitem[Lipton et~al.(2018)Lipton, Wang, and Smola]{lipton2018detecting}
Zachary~C Lipton, Yu-Xiang Wang, and Alex Smola.
\newblock Detecting and correcting for label shift with black box predictors.
\newblock \emph{arXiv preprint arXiv:1802.03916}, 2018.

\bibitem[Long et~al.(2014)Long, Wang, Ding, Sun, and Yu]{long2014transfer}
Mingsheng Long, Jianmin Wang, Guiguang Ding, Jiaguang Sun, and Philip~S Yu.
\newblock Transfer joint matching for unsupervised domain adaptation.
\newblock In \emph{Proceedings of the IEEE conference on computer vision and
  pattern recognition}, pages 1410--1417, 2014.

\bibitem[Long et~al.(2015)Long, Cao, Wang, and Jordan]{long2015learning}
Mingsheng Long, Yue Cao, Jianmin Wang, and Michael Jordan.
\newblock Learning transferable features with deep adaptation networks.
\newblock In \emph{International Conference on Machine Learning}, pages
  97--105, 2015.

\bibitem[Long et~al.(2016)Long, Zhu, Wang, and Jordan]{long2016unsupervised}
Mingsheng Long, Han Zhu, Jianmin Wang, and Michael~I Jordan.
\newblock Unsupervised domain adaptation with residual transfer networks.
\newblock In \emph{Advances in Neural Information Processing Systems}, pages
  136--144, 2016.

\bibitem[Mansour and Schain(2012)]{mansour2012robust}
Yishay Mansour and Mariano Schain.
\newblock Robust domain adaptation.
\newblock In \emph{ISAIM}, 2012.

\bibitem[Mansour et~al.(2009{\natexlab{a}})Mansour, Mohri, and
  Rostamizadeh]{mansour2009domain}
Yishay Mansour, Mehryar Mohri, and Afshin Rostamizadeh.
\newblock Domain adaptation: Learning bounds and algorithms.
\newblock \emph{arXiv preprint arXiv:0902.3430}, 2009{\natexlab{a}}.

\bibitem[Mansour et~al.(2009{\natexlab{b}})Mansour, Mohri, and
  Rostamizadeh]{mansour2009multiple}
Yishay Mansour, Mehryar Mohri, and Afshin Rostamizadeh.
\newblock Multiple source adaptation and the r{\'e}nyi divergence.
\newblock In \emph{Proceedings of the Twenty-Fifth Conference on Uncertainty in
  Artificial Intelligence}, pages 367--374. AUAI Press, 2009{\natexlab{b}}.

\bibitem[Pei et~al.(2018)Pei, Cao, Long, and Wang]{pei2018multi}
Zhongyi Pei, Zhangjie Cao, Mingsheng Long, and Jianmin Wang.
\newblock Multi-adversarial domain adaptation.
\newblock 2018.

\bibitem[Shen et~al.(2018)Shen, Qu, Zhang, and Yu]{shen2018wasserstein}
Jian Shen, Yanru Qu, Weinan Zhang, and Yong Yu.
\newblock Wasserstein distance guided representation learning for domain
  adaptation.
\newblock In \emph{Thirty-Second AAAI Conference on Artificial Intelligence},
  2018.

\bibitem[Shrivastava et~al.(2016)Shrivastava, Pfister, Tuzel, Susskind, Wang,
  and Webb]{shrivastava2016learning}
Ashish Shrivastava, Tomas Pfister, Oncel Tuzel, Josh Susskind, Wenda Wang, and
  Russ Webb.
\newblock Learning from simulated and unsupervised images through adversarial
  training.
\newblock \emph{arXiv preprint arXiv:1612.07828}, 2016.

\bibitem[Tzeng et~al.(2017)Tzeng, Hoffman, Saenko, and
  Darrell]{tzeng2017adversarial}
Eric Tzeng, Judy Hoffman, Kate Saenko, and Trevor Darrell.
\newblock Adversarial discriminative domain adaptation.
\newblock \emph{arXiv preprint arXiv:1702.05464}, 2017.

\bibitem[Zhang et~al.(2013)Zhang, Sch{\"o}lkopf, Muandet, and
  Wang]{zhang2013domain}
Kun Zhang, Bernhard Sch{\"o}lkopf, Krikamol Muandet, and Zhikun Wang.
\newblock Domain adaptation under target and conditional shift.
\newblock In \emph{International Conference on Machine Learning}, pages
  819--827, 2013.

\bibitem[Zhang et~al.(2017)Zhang, Barzilay, and Jaakkola]{zhang2017aspect}
Yuan Zhang, Regina Barzilay, and Tommi Jaakkola.
\newblock Aspect-augmented adversarial networks for domain adaptation.
\newblock \emph{arXiv preprint arXiv:1701.00188}, 2017.

\bibitem[Zhao et~al.(2018{\natexlab{a}})Zhao, Zhang, Wu, Gordon,
  et~al.]{zhao2018multiple}
Han Zhao, Shanghang Zhang, Guanhang Wu, Geoffrey~J Gordon, et~al.
\newblock Multiple source domain adaptation with adversarial learning.
\newblock In \emph{International Conference on Learning Representations},
  2018{\natexlab{a}}.

\bibitem[Zhao et~al.(2018{\natexlab{b}})Zhao, Zhang, Wu, Moura, Costeira, and
  Gordon]{zhao2018adversarial}
Han Zhao, Shanghang Zhang, Guanhang Wu, Jos{\'e}~MF Moura, Joao~P Costeira, and
  Geoffrey~J Gordon.
\newblock Adversarial multiple source domain adaptation.
\newblock In \emph{Advances in Neural Information Processing Systems}, pages
  8568--8579, 2018{\natexlab{b}}.

\bibitem[Zhao et~al.(2019{\natexlab{a}})Zhao, Hu, Zhu, Coates, and
  Gordon]{zhao18deep}
Han Zhao, Junjie Hu, Zhenyao Zhu, Adam Coates, and Geoffrey~J. Gordon.
\newblock Deep generative and discriminative domain adaptation.
\newblock In \emph{Proceedings of the 18th International Conference on
  Autonomous Agents and MultiAgent Systems}, 2019{\natexlab{a}}.

\bibitem[Zhao et~al.(2019{\natexlab{b}})Zhao, Stretcu, Smola, and
  Gordon]{zhao2017efficient}
Han Zhao, Otilia Stretcu, Alex Smola, and Geoff Gordon.
\newblock Efficient multitask feature and relationship learning.
\newblock In \emph{Proceedings of the Thirty-Fifth Conference on Uncertainty in
  Artificial Intelligence}. AUAI Press, 2019{\natexlab{b}}.

\end{thebibliography}
\bibliographystyle{plainnat}

\appendix
\onecolumn
\section{Missing Proofs}
\technical*
\begin{proof}
By definition, for $\forall h, h'\in \HH$, we have:
\begin{align}
|\err_S(h, h') - \err_T(h, h')| &\leq \sup_{h, h'\in\HH} |\err_S(h, h') - \err_T(h, h')| \nonumber\\
&= \sup_{h, h'\in \HH} \big| \Exp_{\xx\sim S}[|h(\xx) - h'(\xx)|] - \Exp_{\xx\sim T}[|h(\xx) - h'(\xx)|]\big| 
\label{equ:k}
\end{align}
Since $||h||_\infty \leq 1, \forall h\in \HH$, then $0\leq |h(\xx) - h'(\xx)| \leq 1$, $\forall \xx\in\data, h, h'\in\HH$. We now use Fubini's theorem to bound $\big| \Exp_{\xx\sim S}[|h(\xx) - h'(\xx)|] - \Exp_{\xx\sim T}[|h(\xx) - h'(\xx)|]\big|$:
\begin{align*}
\big| \Exp_{\xx\sim S}[|h(\xx) - h'(\xx)|] - & \Exp_{\xx\sim T}[|h(\xx) - h'(\xx)|]\big|  \\
 &= \Big| \int_0^1 \left(\Pr_S(|h(\xx) - h'(\xx)| > t) - \Pr_T(|h(\xx) - h'(\xx)| > t)\right)~dt   \Big| \\
 &\leq  \int_0^1 \Big|\Pr_S(|h(\xx) - h'(\xx)| > t) - \Pr_T(|h(\xx) - h'(\xx)| > t)\Big|~dt   \\
 &\leq \sup_{t\in[0, 1]}\Big|\Pr_S(|h(\xx) - h'(\xx)| > t) - \Pr_T(|h(\xx) - h'(\xx)| > t)\Big|
\end{align*}
Now in view of \eqref{equ:k} and the definition of $\tilde{\HH}$, we have:
\begin{align*}
&\sup_{h, h'\in\HH}\sup_{t\in[0, 1]} \Big|\Pr_S(|h(\xx) - h'(\xx)| > t) - \Pr_T(|h(\xx) - h'(\xx)| > t)\Big| \\
= & ~ \sup_{\tilde{h}\in\tilde{\HH}}|\Pr_S(\tilde{h}(\xx) = 1) - \Pr_T(\tilde{h}(\xx) = 1)| \\
= & ~ \sup_{A\in\mathcal{A}_{\tilde{\HH}}}|\Pr_S(A) - \Pr_T(A)| \\
= & ~ d_{\tilde{\HH}}(\domain_S, \domain_T)
\end{align*}
Combining all the inequalities above finishes the proof. 
\end{proof}

\tri*
\begin{proof}
\begin{align*}
    \err_\domain(h, h') &= \Exp_{\xx\sim\domain}[|h(\xx) - h'(\xx)|] = \Exp_{\xx\sim\domain}[|h(\xx) - h''(\xx) + h''(\xx) - h'(\xx)|] \\
    & \leq \Exp_{\xx\sim\domain}[|h(\xx) - h''(\xx)| + |h''(\xx) - h'(\xx)|] = \err_\domain(h, h'') + \err_\domain(h'', h')
\end{align*}
\end{proof}

\population*
\begin{proof}
On one hand, with Lemma~\ref{lemma:key} and Lemma~\ref{lemma:tri}, we have $\forall h\in\HH$:
\begin{equation*}
    \err_T(h) = \err_{T}(h, f_T) \leq \err_{S}(h, f_T) + d_{\tilde{H}}(\domain_S, \domain_T)  \leq \err_{S}(h) + \err_{S}(f_S, f_T) + d_{\tilde{H}}(\domain_S, \domain_T).
\end{equation*}
On the other hand, by changing the order of two triangle inequalities, we also have:
\begin{equation*}
    \err_T(h) = \err_{T}(h, f_T) \leq \err_{T}(h, f_S) + \err_{T}(f_S, f_T) \leq \err_S(h) + \err_{T}(f_S, f_T) + d_{\tilde{H}}(\domain_S, \domain_T).
\end{equation*}
Realize that by definition $\err_S(f_S, f_T) = \Exp_{\domain_S}[|f_S - f_T|]$ and $\err_T(f_S, f_T) = \Exp_{\domain_T}[|f_S - f_T|]$. Combining the above two inequalities completes the proof.
\end{proof}

\source*
\begin{proof}
Consider the source domain $\domain_S$. For $\forall h\in\HH$, define the loss function $\ell: \data\to[0,1]$ as $\ell(\xx)\defeq |h(\xx) - f_S(\xx)|$. First, we know that $\emrad_\sample(\HH - f_S) = \emrad_\sample(\HH)$ where we slightly abuse the notation $\HH - f_S$ to mean the family of functions $\{h - f_S\mid \forall h\in\HH\}$:
\begin{align*}
    \emrad_\sample(\HH - f_S) &= \Exp_{\sigmas}\bigg[\sup_{h'\in\HH - f_S}\frac{1}{n}\sum_{i=1}^n \sigma_i h'(\xx_i)\bigg] = \Exp_{\sigmas}\bigg[\sup_{h\in\HH}\frac{1}{n}\sum_{i=1}^n \sigma_i (h(\xx_i) - f_S(\xx_i))\bigg] \\
    &= \Exp_{\sigmas}\bigg[\sup_{h\in\HH}\frac{1}{n}\sum_{i=1}^n \sigma_i h(\xx_i)\bigg] + \Exp_{\sigmas}\bigg[\frac{1}{n}\sum_{i=1}^n \sigma_i f_S(\xx_i)\bigg] \\
    &= \emrad_\sample(\HH)
\end{align*}
Observe that the function $\phi: t\to |t|$ is 1-Lipschitz continuous, then by Ledoux-Talagrand’s contraction lemma, we can conclude that 
\begin{equation*}
    \emrad_\sample(\phi \circ (\HH - f_S))\leq \emrad_\sample(\HH - f_S) = \emrad_\sample(\HH)
\end{equation*}
Using Lemma~\ref{lemma:emrad} with the above arguments and realize that $\err_S(h) = \Exp_{\xx\sim\domain_S}[|h(\xx) - f_S(\xx)|]$ finishes the proof.
\end{proof}

\ddd*
\begin{proof}
Note that $\ind_h\in\{0, 1\}$, hence this lemma directly follows Lemma~\ref{lemma:emrad}.
\end{proof}

\hdiv*
\begin{proof}
By the triangular inequality of $d_{\tilde{\HH}}(\cdot, \cdot)$, we have:
\begin{equation*}
    d_{\tilde{\HH}}(\domain, \domain') \leq d_{\tilde{\HH}}(\domain, \empdomain) + d_{\tilde{\HH}}(\empdomain, \empdomain') + d_{\tilde{\HH}}(\empdomain', \domain').
\end{equation*}
Now with Lemma~\ref{lemma:ind}, we know that with probability $\geq 1-\delta / 2$, we have:
\begin{equation*}
    d_{\tilde{\HH}}(\domain, \empdomain) \leq 2\emrad_\sample({\tilde{\HH}}) + 3\sqrt{\log(4/\delta)/2n}.
\end{equation*}
Similarly, with probability $\geq 1-\delta/2$, the following inequality also holds:
\begin{equation*}
    d_{\tilde{\HH}}(\domain', \empdomain') \leq 2\emrad_\sample({\tilde{\HH}}) + 3\sqrt{\log(4/\delta)/2n}.
\end{equation*}
A union bound to combine the above two inequalities then finishes the proof. 
\end{proof}

\main*
\begin{proof}
By Theorem~\ref{thm:populationbound}, the following inequality holds:
\begin{equation*}
    \err_T(h) \leq \err_S(h) + d_{\tilde{\HH}}(\domain_S, \domain_T) + \min\{\Exp_{\domain_S}[|f_S - f_T|], \Exp_{\domain_T}[|f_S - f_T|]\}.
\end{equation*}
To get probabilistic bounds for both $\err_S(h)$ and $d_{\tilde{\HH}}(\domain_S, \domain_T)$, we apply Lemma~\ref{lemma:source} and Lemma~\ref{lemma:hdiv}, respectively. The final step, again, is to use a union bound to combine all the inequalities above, which completes the proof.
\end{proof}

\dpi*
\begin{proof}
Let $B$ be a uniform random variable taking value in $\{0,1\}$ and let the random variable $Y_B$ with distribution $\domain_B^Y$ (resp. $Z_B$ with distribution $\domain_B^Z$) be the mixture of $\domain_S^Y$ and $\domain_T^Y$ (resp. $\domain_S^Z$ and $\domain_T^Z$) according to $B$. We know that: 
\begin{equation}
    \jsd(\domain_S^Z~||~\domain_T^Z) = I(B; Z_B),\quad \text{and} \quad \jsd(\domain_S^Y~||~\domain_T^Y) = I(B; Y_B).
\end{equation}
Since $\domain_S^Y$ (resp. $\domain_T^Y$) is induced by the function $h:\hdata\mapsto\ydata$ from $\domain_S^Z$ (resp. $\domain_T^Z$), by linearity, we also have $\domain_B^Y$ is induced by $h$ from $\domain_B^Z$. Hence $Y_B = h(Z_B)$ and the following Markov chain holds:
\begin{equation*}
    B\rightarrow Z_B\rightarrow Y_B.
\end{equation*}
Apply the data processing inequality (Lemma~\ref{lemma:dpi}), we have
\begin{equation*}
    \jsd(\domain_S^Z~||~\domain_T^Z) = I(B;Z_B)\geq I(B; Y_B) = \jsd(\domain_S^Y~||~\domain_T^Y).
\end{equation*}
Taking square root on both sides of the above inequality completes the proof.
\end{proof}

\relationship*
\begin{proof}
\begin{align*}
    \djs(\domain^Y, \domain^{\hat{Y}}) &= \sqrt{\jsd(\domain^Y, \domain^{\hat{Y}})} \\
    & \leq \sqrt{||\domain^Y - \domain^{\hat{Y}}||_1/2} && (\text{Lemma~\ref{lemma:lin}})\\
    &= \sqrt{\left(|\Pr(Y = 0) - \Pr(\hat{Y} = 0)| + |\Pr(Y = 1) - \Pr(\hat{Y} = 1)|\right)/2}\\
    &= \sqrt{|\Pr(Y = 1) - \Pr(\hat{Y} = 1)|} \\
    &= \sqrt{|\Exp_X[f(X)] - \Exp_X[h(g(X))]|} \\
    &\leq \sqrt{\Exp_X[|f(X) - h(g(X))|]} \\
    &= \sqrt{\err(h\circ g)}
\end{align*}
\end{proof}

\lowerbound*
\begin{proof}
Since $X \overset{g}{\longrightarrow} Z \overset{h}{\longrightarrow} \hat{Y}$ forms a Markov chain, by Lemma~\ref{lemma:jsd}, the following inequality holds:
\begin{equation*}
\djs(\domain^{\hat{Y}}_S, \domain^{\hat{Y}}_T) \leq \djs(\domain^Z_S, \domain^Z_T).    
\end{equation*}
On the other hand, since $\djs(\cdot, \cdot)$ is a distance metric, we also have:
\begin{equation*}
    \djs(\domain_S^Y, \domain_T^Y) \leq \djs(\domain_S^Y, \domain_S^{\hat{Y}}) + \djs(\domain_S^{\hat{Y}}, \domain_T^{\hat{Y}}) + \djs(\domain_T^{\hat{Y}}, \domain_T^Y) \leq \djs(\domain_S^Y, \domain_S^{\hat{Y}}) + \djs(\domain^Z_S, \domain^Z_T) + \djs(\domain_T^{\hat{Y}}, \domain_T^Y).
\end{equation*}
Applying Lemma~\ref{lemma:relationship} to both $\djs(\domain_S^Y, \domain_S^{\hat{Y}})$ and $\djs(\domain_T^{\hat{Y}}, \domain_T^Y)$ then finishes the proof.
\end{proof}

\ll*
\begin{proof}
In view of the result in Theorem~\ref{thm:lower}, applying the AM-GM inequality, we have:
\begin{equation*}
\sqrt{\eps_S(h\circ g)} + \sqrt{\eps_T(h\circ g)} \leq \sqrt{2\left(\eps_S(h\circ g) + \eps_T(h\circ g)\right)}.
\end{equation*}
Now since $\djs(\domain_S^Y, \domain_T^Y)\geq \djs(\domain_S^Z, \domain_T^Z)$, simple algebra shows
\begin{equation*}
    \err_S(h\circ g) + \err_T(h\circ g) \geq \frac{1}{2}\left(\djs(\domain_S^Y, \domain_T^Y) - \djs(\domain_S^Z, \domain_T^Z)\right)^2.
\end{equation*}
\end{proof}

\section{Technical Tools}
The following lemma is particularly useful to provide data-dependent guarantees in terms of the empirical Rademacher complexity:
\begin{lemma}[\citet{bartlett2002rademacher}]
Let $\HH\subseteq [0, 1]^\data$, then for $\forall \delta > 0$, w.p.b. at least $1-\delta$, the following inequality holds for $\forall h\in\HH$:
\begin{equation}
    \Exp[h(\xx)] \leq \frac{1}{n}\sum_{i=1}^n h(\xx_i) + 2\emrad_\sample(\HH) + 3\sqrt{\frac{\log(2/\delta)}{2n}}
\end{equation}
\label{lemma:emrad}
\end{lemma}
Ledoux-Talagrand's contraction lemma is a useful technique in upper bounding the Rademacher complexity of function compositions:
\begin{lemma}[Ledoux-Talagrand's contraction lemma]
Let $\phi:\RR\mapsto\RR$ be a Lipschitz function with parameter $L$, i.e., $\forall a,b\in\RR$, $|\phi(a) - \phi(b)|\leq L|a-b|$. Then, 
\begin{equation*}
    \emrad_\sample(\phi\circ \HH) = \Exp_{\sigmas}\bigg[\sup_{h\in\HH}\frac{1}{n}\sum_{i=1}^n \sigma_i \phi(h(\xx_i))\bigg] \leq L~\Exp_{\sigmas}\bigg[\sup_{h\in\HH}\frac{1}{n}\sum_{i=1}^n \sigma_i h(\xx_i)\bigg] = L~\emrad_\sample(\HH),
\end{equation*}
where $\phi\circ\HH\defeq\{\phi\circ h~\mid~ h\in\HH\}$ is the class of composite functions.
\end{lemma}

Lin's lemma gives an upper bound of the JS divergence between two distributions via the $L_1$ distance (total variation distance).
\begin{lemma}[Theorem. 3,~\citep{lin1991divergence}]
\label{lemma:lin}
Let $\domain$ and $\domain'$ be two distributions, then $\jsd(\domain, \domain')\leq \frac{1}{2}||\domain - \domain'||_1$.
\end{lemma}

\begin{lemma}[Data processing inequality]
Let $X\rightarrow Z\rightarrow Y$ be a Markov chain, then $I(X;Z)\geq I(X; Y)$, where $I(\cdot;\cdot)$ is the mutual information.
\label{lemma:dpi}
\end{lemma}

\section{Additional Experiments}
\label{sec:additional}
In order to further validate our claims, we artificially unbalance the label distribution on the source domain by removing samples from the dataset. We perform two such modifications:
\begin{itemize}
    \item \textbf{Unbalanced digits} In our first experiment, the source domain is MNIST, from which we randomly remove $70\%$ of the first five classes (corresponding to digits $0$ through $4$) while leaving the other classes untouched. The target domain is the full USPS dataset.
    \item \textbf{Unbalanced zeros and ones} In our second experiment, the source domain is still MNIST. We remove $70\%$ of the $0$ class and all the classes above $2$ entirely. We still target the USPS dataset, but also remove digits $2$ to $9$ in that dataset.
\end{itemize}
The results of the DANN domain adaptation algorithm on those tasks are plotted in Figure~\ref{fig:dann2}. They confirm the theoretical and experimental findings from the main text. The effect is however enhanced due to a much larger discrepancy between the label distributions (a fact predicted by our theory). Those plots are the mean across 5 seeds, the standard deviation over those 5 runs is significantly lower than the observed trend.

\begin{figure*}[htb]
    \centering
    \begin{subfigure}[b]{0.45\linewidth}
        \includegraphics[width=\linewidth]{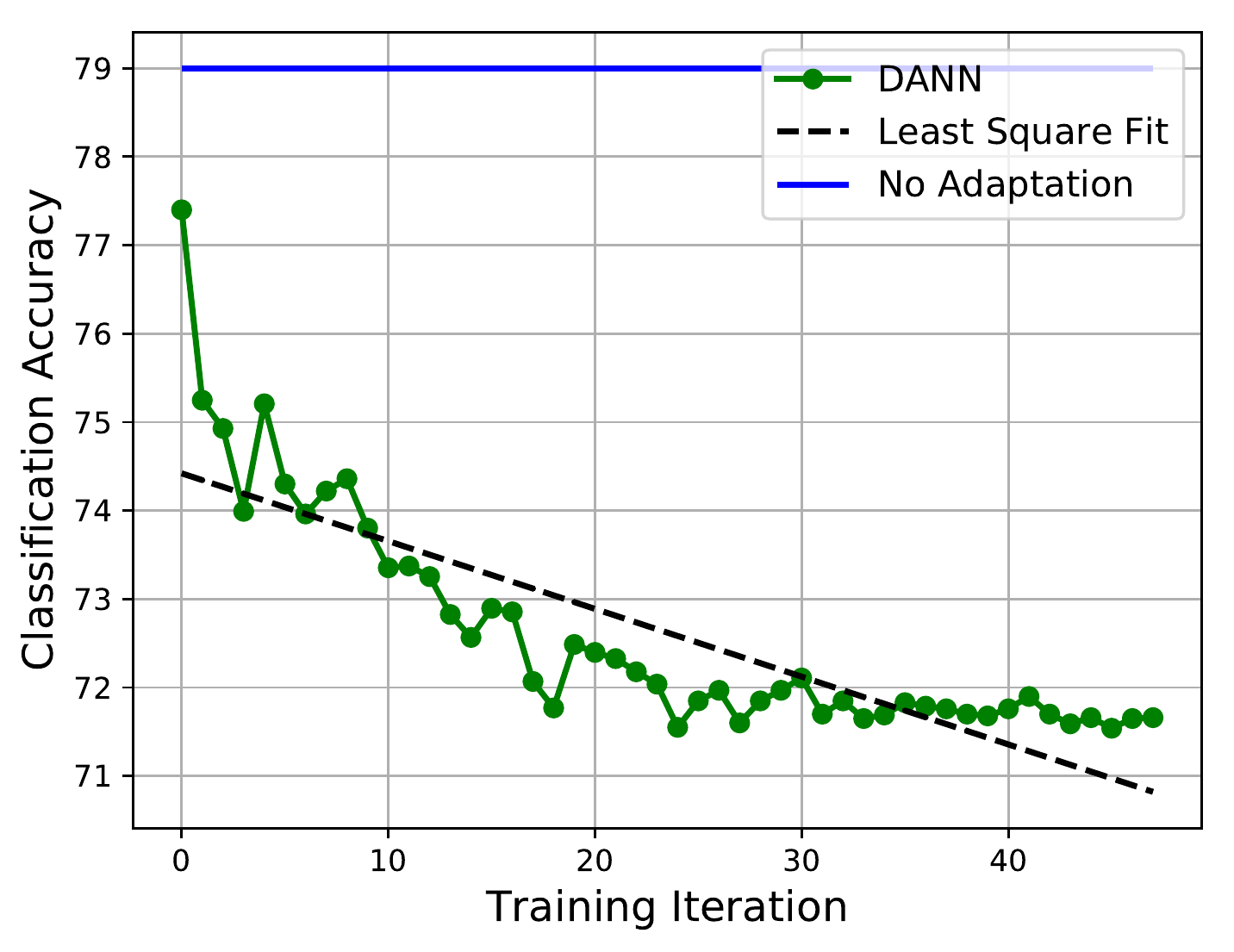}
        \caption{Unbalanced digits}
    \end{subfigure}
    ~
    \begin{subfigure}[b]{0.45\linewidth}
        \includegraphics[width=\linewidth]{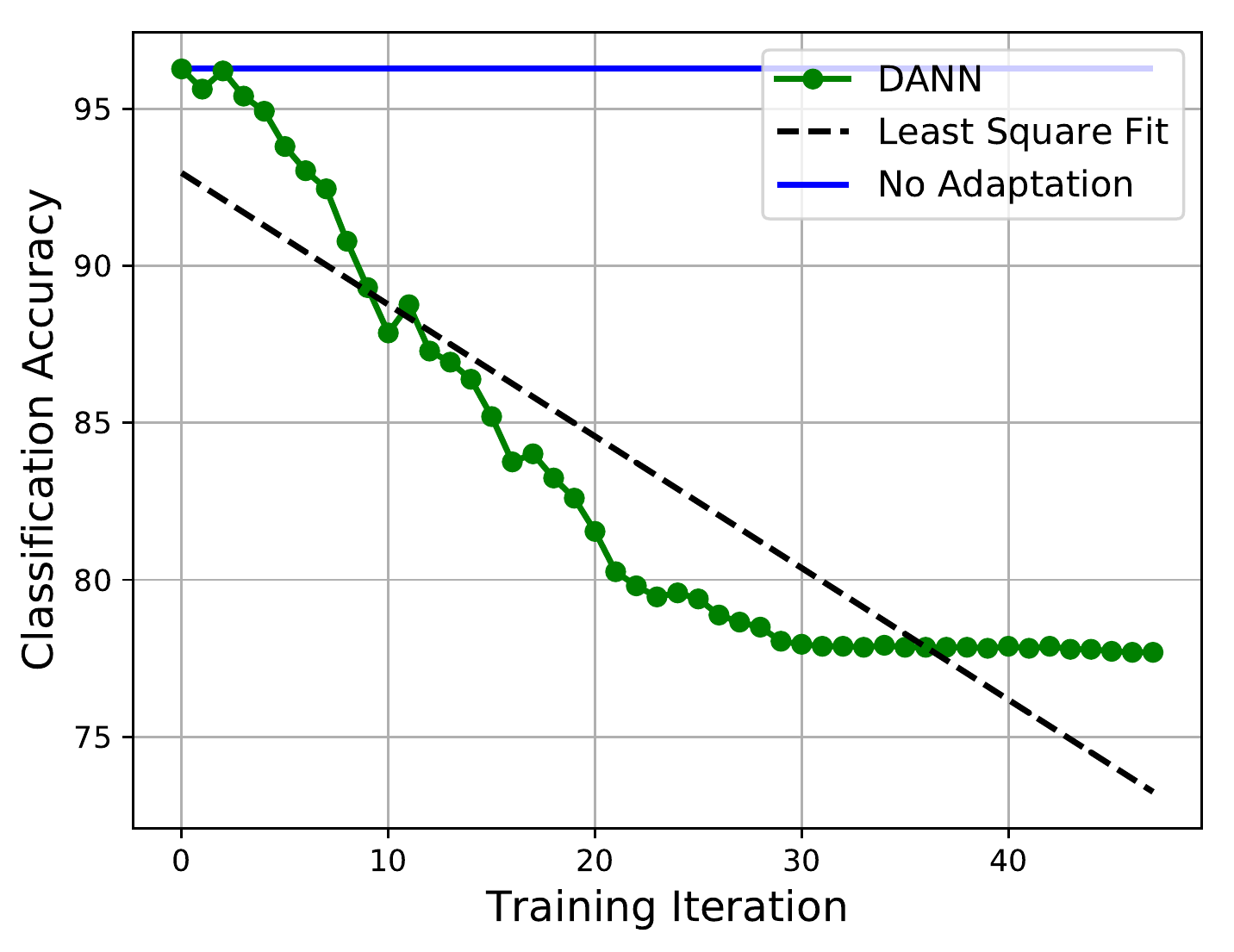}
        \caption{Unbalanced zeros and ones}
    \end{subfigure}
\caption{Digit classification on the unbalanced MNIST to USPS domain adaptation tasks described above. The horizontal solid line corresponds to the target domain test accuracy without adaptation. The green solid line is the target domain test accuracy under domain adaptation with DANN. We also plot the least square fit (dashed line) of the DANN adaptation results to emphasize the negative slope.}
\label{fig:dann2}
\end{figure*}

\end{document}